\documentclass{article}

    \usepackage[final,nonatbib]{neurips_2022}

\usepackage{wrapfig} 
\usepackage[utf8]{inputenc} % allow utf-8 input
\usepackage[T1]{fontenc}    % use 8-bit T1 fonts
\usepackage{hyperref}       % hyperlinks
\usepackage{url}            % simple URL typesetting
\usepackage{booktabs}       % professional-quality tables
\usepackage{amsfonts}       % blackboard math symbols
\usepackage{nicefrac}       % compact symbols for 1/2, etc.
\usepackage{microtype}      % microtypography

\usepackage{cite}
\usepackage{color}
\usepackage{latexsym}
\usepackage{amsmath}
\usepackage{amsthm}
\usepackage{amssymb}
\usepackage{algorithm}
\usepackage{algorithmic}
\usepackage{graphics} % for pdf, bitmapped graphics files
\usepackage{epsfig} % for postscript graphics files

\usepackage{enumitem}
\usepackage{multirow}
\usepackage{array}
\usepackage{bbding}
\usepackage{threeparttable}
\usepackage{thmtools, thm-restate}

\usepackage[table]{xcolor}
\usepackage{graphicx,subfigure}

\usepackage{booktabs}
\usepackage{makecell}
\usepackage{bigstrut,rotating}
\usepackage{soul}
\usepackage[textsize=tiny]{todonotes}
\usepackage{pifont}

\newcommand{\alg}{$\mathsf{FSGDA}~$}
\newcommand{\algn}{$\mathsf{SAGDA}~$}
\newcommand{\algnns}{$\mathsf{SAGDA}$}

\renewcommand{\a}{\mathbf{a}}

\renewcommand{\b}{\mathbf{b}}

\newcommand{\e}{\mathbf{e}}

\renewcommand{\u}{\mathbf{u}}

\renewcommand{\v}{\mathbf{v}}

\newcommand{\w}{\mathbf{w}}

\newcommand{\x}{\mathbf{x}}

\newcommand{\y}{\mathbf{y}}
\newcommand{\z}{\mathbf{z}}

\newcommand{\mc}[1]{\mathcal{#1}}
\newcommand{\mb}[1]{\mathbb{#1}}

\newtheorem{defn}{Definition}

\newtheorem{assum}{Assumption}
\title{SAGDA: Achieving $\mathcal{O}(\epsilon^{-2})$ Communication Complexity in Federated Min-Max Learning}

\author{
Haibo Yang \\
Dept. of ECE\\
The Ohio State University\\
Columbus, OH 43210 \\
\texttt{yang.5952@osu.edu} \\
\And
Zhuqing Liu \\
Dept. of ECE\\
The Ohio State University\\
Columbus, OH 43210 \\
\texttt{liu.9384@osu.edu} \\ 
\AND
Xin Zhang \\
Dept. of Statistics\\
Iowa State University\\
Ames, IA 50010 \\
\texttt{xinzhang@iastate.edu} \\
\And
Jia Liu \\
Dept. of ECE\\
The Ohio State University\\
Columbus, OH 43210 \\
\texttt{liu@ece.osu.edu} \\ 
}

\begin{document}

\maketitle
\allowdisplaybreaks

\begin{abstract}
Federated min-max learning has received increasing attention in recent years thanks to its wide range of applications in various learning paradigms.
Similar to the conventional federated learning for empirical risk minimization problems, communication complexity also emerges as one of the most critical concerns that affects the future prospect of federated min-max learning.
To lower the communication complexity of federated min-max learning, a natural approach is to utilize the idea of infrequent communications (through multiple local updates) same as in conventional federated learning.
However, due to the more complicated inter-outer problem structure in federated min-max learning, theoretical understandings of communication complexity for federated min-max learning with infrequent communications remain very limited in the literature.
This is particularly true for settings with non-i.i.d. datasets and partial client participation.
To address this challenge, in this paper, we propose a new algorithmic framework called \ul{s}tochastic \ul{s}ampling \ul{a}veraging \ul{g}radient \ul{d}escent \ul{a}scent (\algnns), which i) assembles stochastic gradient estimators from randomly sampled clients as control variates  and ii) leverages two learning rates on both server and client sides.
We show that \algn achieves a linear speedup in terms of both the number of clients and local update steps, which yields an $\mathcal{O}(\epsilon^{-2})$ communication complexity that is orders of magnitude lower than the state of the art.
Interestingly, by noting that the standard federated stochastic gradient descent ascent (FSGDA) is in fact a control-variate-free special version of \algnns, we immediately arrive at an $\mathcal{O}(\epsilon^{-2})$ communication complexity result for FSGDA.
Therefore, through the lens of \algnns, we also advance the current understanding on communication complexity of the standard FSGDA method for federated min-max learning.
%we first propose a \ul{f}ederated \ul{s}tochastic \ul{g}radient \ul{d}escent \ul{a}scent (FSGDA) method, which achieves the same sub-linear convergence rate as that in centralized min-max learning and a linear speedup in terms of number of participated clients.
%To further increase the allowed number of local updates (hence being more communication-efficient), we further propose a variance-reduced variant called \ul{s}tochastic \ul{s}ampling \ul{a}veraging \ul{g}radient \ul{d}escent \ul{a}scent (\algnns) to achieve a linear speedup in terms of both the number of clients and local steps and with significantly fewer communication rounds under arbitrary data heterogeneity.
%Extensive numerical experiments corroborate the effectiveness and efficiency of our algorithms.
\end{abstract}
% !TEX root = main.tex

\section{Introduction} \label{sec: Introduction}
Recently, min-max optimization has drawn considerable attention from the machine learning community.
Compared with conventional minimization problems (e.g., empirical risk minimization), min-max optimization has a richer mathematical structure, thus being able to model more sophisticated learning problems that emerge from ever-emerging applications.
In particular, the subclass of nonconvex-concave and nonconvex-PL (Polyak-Łojasiewicz) min-max problems has important applications in, e.g., AUC (area under the ROC curve) maximization~\cite{ying2016stochasticAUC,liu2019stochasticAUC}, adversarial and robust learning~\cite{madry2018adversarial,sinha2018certifying}, and generative adversarial network (GAN)~\cite{goodfellow2014gan}.
The versatility of min-max optimization thus sparks intense research on developing efficient min-max algorithms.
% Since the seminal work~\cite{wald1945statistical}, a large amount of works have been pursuing the efficient algorithms to tackle min-max optimization.
In the literature, the family of primal-dual stochastic gradient methods is one of the most popular and efficient approaches.
For example, the stochastic gradient descent ascent (SGDA) method in this family has been shown effective in centralized (single-machine) learning, both theoretically and empirically.
%With low iteration complexity, SGDA is more suitable for deep learning applications.
%Thus, it inspires many follow-ups to push the boundary of first-order optimization methods in min-max problems.
However, as over-parameterized models (e.g., deep neural networks) being more and more prevalent, learning on a single machine becomes increasingly inefficient.
%shows the shortage of computation efficiency.
To address challenge, large-scale distributed learning emerges as an effective mechanism to accelerate training and has achieved astonishing successes in recent years.
Moreover, as more stringent data privacy requirements arise in recent years, centralized learning becomes increasingly infeasible due to the prohibition of data collection.
This also motivates the need for distributed learning without sharing raw data.
Consequently, there is a growing need for distributed/federated min-max optimization, such as federated deep AUC maximization~\cite{guo2020communicationAUC,yuan2021federatedAUC}, federated adversarial training~\cite{reisizadeh2020robust} and distributed/federated GAN~\cite{brock2018largeGAN,augenstein2019generative,rasouli2020fedgan}.

Similar to conventional federated learning for minimization problems, federated min-max learning enjoys benefits of parallelism and privacy, but suffers from high communication costs.
%Current communication channel is usually bandwidth-limited and thus not able to afford these expensive cost when large-scale model is used and frequent information exchange is required.
%This situation deteriorates for large-scale learning system and wireless connection such as that in IoT devices.
%So how to design efficient algorithm with low communication complexity in large-scale min-max learning is the centrality to underpin many real-world applications.
One effective approach to reduce communication costs is to utilize {\em infrequent} communications. %between multiple local update steps at each client.
For example, in conventional federated learning for minimization problems, the FedAvg algorithm\cite{mcmahan2016communication}  allows each client performs multiple stochastic gradient descent (SGD) steps to update the local model between two successive communication rounds. % updates via SGD.
Then, local models are sent to and averaged periodically
% through infrequent communication between client and 
 at the server through communications.
Although infrequent communication may introduce extra noises due to {\em data heterogeneity}, %{\em local update steps} and even {\em partial client participation}, 
FedAvg can still achieve the same convergence rate as distributed SGD, while having a significant lower communication complexity.
%More specifically, FedAvg achieves $\mathcal{O}(1/\sqrt{mKT})$ convergence rate with proper learning rate in non-convex optimization, showing the linear speedup in both number of clients $m$ and local steps $K$.
Inspired by the theoretical and empirical success of FedAvg, a natural idea to lower the communication costs of federated min-max optimization is to utilize infrequent communication in the federated version of SGDA.
%as a natural extension for federated min-max learning, the primal and dual variables could be trained by SGDA locally with infrequent communication.
%Despite the simplicity of such algorithm, the thorough understanding on non-convex-concave (or non-convex-PL) min-max problem still lacks theoretical foundations.
Despite the simplicity of this idea, existing works can only show unsatisfactory convergence rates ($\mathcal{O}(1/\sqrt{mT})$~\cite{xie2021federated} and $\mathcal{O}(1/\left(mKT\right)^{1/3})$~\cite{deng2021local}) for solving non-convex-strongly-concave or non-convex-PL by federated SGDA with infrequent communication ($m$ is the number of clients, $K$ is the number of local steps, and T is the number of communication rounds).
These convergence rates do not match with that of the FedAvg method.
%However, none of them matches the desired rate as that in FedAvg.
These unsatisfactory results are due to the fact that federated min-max optimization not only needs to address the same challenges in conventional federated learning (e.g., data heterogeneity and partial client participation), but also handle the more complicated inter-outer problem structure.
Thus, a fundamental question in federated min-max optimization is: {\em Can a federated SGDA-type method with infrequent communication provably achieve the same convergence rate and even the highly desirable linear speedup effect for federated min-max problems?}
%linear speedup for federated non-convex-PL min-max problem with heterogeneous data?}

In this paper, we answer this question affirmatively.
%For non-convex-PL min-max optimization, we show the federated SGDA algorithm can also achieve the same speedup and thus having a significant improvement over the state-of-the-art results.
The main contributions of this paper are summarized as follows:
\begin{list}{\labelitemi}{\leftmargin=0.8em \itemindent=-0.0em \itemsep=.1em}
    %\vspace{-.05in}
\item 
We propose a new algorithmic framework called \algnns (\ul{s}tochastic \ul{s}ampling \ul{a}veraging \ul{g}radient \ul{d}escent \ul{a}scent), which assembles stochastic gradient estimators as control variates and leverages two learning rates on both server and client sides.
With these techniques, \algn {\em relaxes} the restricted ``bounded gradient dissimilarity'' assumption, while still achieving the same convergence rate with low communication complexity. 
We show that \algn achieves the highly desirable linear speedup in terms of both the number of clients (even with partial client participation) and local update steps, which yields an $\mathcal{O}(\epsilon^{-2})$ communication complexity that is orders of magnitude lower than the state of the art in the literature of federated min-max optimization.
%Moreover, \algn achieves the highly desirable linear speedup effect even with partial client participation (in terms of both number of participated clients $m$ and local step $K$) and under arbitrary data heterogeneity.

\item
Interestingly, by noting that the standard federated stochastic gradient descent ascent (FSGDA) is in fact a ``control-variant-free'' special version of our \algn algorithm, we can conclude from our theoretical analysis of \algn that FSGDA achieves an $\mathcal{O}(1/\sqrt{mKT})$ convergence rate for non-convex-PL problems with full client participation, which further implies the highly desirable linear speedup effect.
%and under the mild data heterogeneity assumption (bounded gradient dissimilarity)
This improves the state-of-the-art result of FSGDA by a factor of $\mathcal{O}(1/\left(mKT\right)^{1/6})$~\cite{deng2021local} and matches the optimal convergence rate of non-convex FL.
Therefore, through the lens of \algnns, we also advance the current understanding on the communication complexity of the standard FSGDA method for federated min-max learning.

%We propose a \ul{f}ederated \ul{s}tochastic \ul{g}radient \ul{d}escent \ul{a}scent (FSGDA) method by generalizing SGDA to the federated min-max setting.
%One novelty of our FSGDA method is that it employs different learning rates on the server and client sides to facilitate a faster convergence rate.
%Under the mild data heterogeneity assumption (bounded gradient dissimilarity), FSGDA achieves an $\mathcal{O}(1/\sqrt{mKT})$ convergence rate for non-convex-PL problems with full client participation, which further implies the highly desirable linear speedup effect.
%This improves the state-of-the-art by a factor of $\mathcal{O}(1/\left(mKT\right)^{1/6})$~\cite{deng2021local} and matches the optimal convergence rate of non-convex FL.

%\item Inspired by variance reduction techniques, we further propose the \ul{f}ederated \ul{a}veraging \ul{s}tochastic \ul{g}radient \ul{d}escent \ul{a}scent (FASGDA) algorithm.
%Thanks to the stochastic averaging from all clients, FASGDA further relaxes the bounded gradient dissimilarity assumption and achieves the same convergence rate with low communication complexity, and the linear speedup effect even with partial client participation (in terms of both number of participated clients $m$ and local step $K$) and under arbitrary data heterogeneity.
\end{list}

\begin{table}[htbp]
	\centering
	\caption{Number of communication rounds and stochastic gradients per client to reach $\epsilon$-stationary point ($\| \nabla \Phi \|  \leq \epsilon$) for federated non-convex-PL min-max learning, denoted as communication and client sample complexity.
    We omit the higher orders.
    Here $m$ is the number of clients.
    BGD means bounded gradient dissimilarity, which requires bounded data heterogeneity. 
    \algn supports client sampling and does not require BGD assumption.
    }
    \renewcommand{\arraystretch}{1.3}
	\begin{tabular}{p{2.7cm}<{\centering}|p{1.6cm}<{\centering} p{2.cm}<{\centering}|p{2.88cm}<{\centering} p{2.15cm}<{\centering}}
		\hline
        \multirow{2}{*}{ Methods} & \multirow{2}{*}{\makecell{BGD \\ Assumption}} & \multirow{2}{*}{\makecell{Client \\ Sampling?}} & {Per-Client Sample}  & { Communication} \\ 
        & & & {Complexity} & {Complexity} \\
        \hline 
		SGDA & $-$ & $-$ & { $\epsilon^{-4}$} & $-$ \\
        \hline
        Local SGDA~\cite{deng2021local} & \ding{52} & \ding{55} & $\max \{ \epsilon^{-4}, \frac{1}{m^{2/3}}\epsilon^{-6} \}$ & $\mathcal{O}((1/m) \epsilon^{-6})$ \\
        % \hline
        \makecell{(Momentum) \\ Local SGDA~\cite{sharma2022federated}} & \ding{52} & \ding{55} & $\mathcal{O}((1/m) \epsilon^{-4})$ & $\mathcal{O}(\epsilon^{-3})$ \\
        % \hline
        CD-MAGE~\cite{xie2021federated} & \ding{52} & \ding{52} & $\mathcal{O}((1/m) \epsilon^{-4})$ & $\mathcal{O}((1/m) \epsilon^{-4})$ \\
        \hline
        \rowcolor{lightgray!50}
		\algn (Cor. ~\ref{cor: sagda}) & Not needed & \ding{52} & $\boldsymbol{\mathcal{O}((1/m) \epsilon^{-4})}$ & { $\boldsymbol{\mathcal{O}(\epsilon^{-2})}$}  \\
        \rowcolor{lightgray!50}
		\alg (Cor. ~\ref{cor: FSGDA_1}~\ref{cor: FSGDA_2}) & \ding{52} & \ding{52} & $\boldsymbol{\mathcal{O}((1/m) \epsilon^{-4})}$ & $\boldsymbol{\mathcal{O}(\epsilon^{-2})}$ \\
        \hline
	\end{tabular}
    \label{tab:bound}
\end{table}

%$\frac{1}{m \epsilon^4}$ 

The rest of the paper is organized as follows.
In Section~\ref{sec: RelatedWork}, we review related work.
In Section~\ref{sec: alg}, we first introduce \algn And its convergence analysis, and then build the connection between \algn and FSGDA.
We present numerical results in Section~\ref{sec: Experiment} and conclude the work in Section~\ref{sec: Conclusion}. 
Due to space limitation, we relegate all proofs and some experiments to the supplementary material.
% !TEX root = main.tex

\section{Related work} \label{sec: RelatedWork}

\textbf{1) Federated Learning:}
In federated learning (FL), the seminal federated averaging (FedAvg)~\cite{mcmahan2017communication} algorithm was first proposed as a heuristic to improve communication efficiency and data privacy, but later theoretically confirmed to achieve a highly desirable $\mathcal{O}(1/\sqrt{mKT})$ convergence rate in FL (implying linear convergence speedup as the number of clients $m$ increases).
%From optimization perspective, FedAvg can be viewed as an extension of SGD to federated learning with local update steps.
%As we know, SGD achieves $\mathcal{O}(1/\sqrt{T})$ convergence rate under proper learning rate in non-convex optimization.
%Paralleled with multiple clients and local steps in federated learning, the linear speedup of convergence rate in terms of number of clients $m$ and local steps $K$ is expected, i.e., $\mathcal{O}(1/\sqrt{mKT})$.
%This indicates that, to achieves $\epsilon$-stationary point ($\nabla \| f(x) \| \leq \epsilon$), the complexity is reduced from $1/\epsilon^4$ to $1/(mK \epsilon^4)$.
%Despite of data heterogeneity and potentially partial client participation, 
Since then, many follow-up works have been proposed to achieve the $\mathcal{O}(1/\sqrt{mKT})$ convergence rate for i.i.d. datasets~\cite{stich2018local,yu2019parallel,wang2018cooperative,stich2019error,lin2018don,khaled2019better,zhou2017convergence} and non-i.i.d. datasets~\cite{sattler2019robust,zhao2018federated,li2018federated,wang2019slowmo,Karimireddy2020SCAFFOLD,huang2018loadaboost,jeong2018communication, yang2021linearspeedup,khanduri2021stem,yang2022anarchic}.
For a comprehensive survey on FL convergence rate order, we refer readers to Section~3 in \cite{kairouz2019advances}.

\textbf{2) Min-max Optimization:}
Min-max optimization has a long history dating back to at least ~\cite{neumann1928theorie,wald1945statistical}.
For non-convex-strongly-concave min-max problems, a simple approach is the stochastic gradient descent ascent (SGDA), which performs stochastic gradient descent on primal variables and stochastic gradient ascent on dual variables, respectively.
It is well-known that SGDA achieves an $\mathcal{O}(1/\sqrt{T})$ convergence rate~\cite{rafique2021weakly, lin2020gradient} for non-convex-strongly-concave min-max problems, matching that of SGD in non-convex optimization.
However, in the federated non-convex-strongly-concave setting, studies in~\cite{xie2021federated} and \cite{deng2021local} only proved $\mathcal{O}(1/\sqrt{mT})$ and $\mathcal{O}(1/(mKT)^{1/3})$ convergence rates, respectively.
So far, it remains unknown whether federated SGDA could achieve the same desirable convergence rate of $\mathcal{O}(1/\sqrt{mKT})$ as FedAvg. 
In this paper, we show that our \algn algorithm and FSGDA (implied by \algnns) indeed achieve the $\mathcal{O}(1/\sqrt{mKT})$ convergence rate, matching that of FedAvg.
%existing works have shown the possibility of distributed/federated min-max optimization~\citet{mateos2015distributed,xie2021federated,deng2021local}.

% \input{Sec3_Preliminary}
% !TEX root = main.tex

\section{Problem statement and algorithm design} \label{sec: alg}

We consider a general min-max optimization problem in federated learning setting as follows:
\begin{align} \label{eqn:minmax}
    \min_{\x \in \mb{R}^d} \max_{\y \in \mb{R}^d} f(\x, \y) := \min_{\x \in \mb{R}^d} \max_{\y \in \mb{R}^d} \frac{1}{M} \sum_{i \in [M]} f_i(\x, \y),
\end{align}
where $f_i(\x, \y) := \mb{E}_{\xi_i \sim D_i} [f(\x, \y, \xi_i)]$ is the local loss function associated with a local data distribution $D_i$ and $M$ is the number of workers.
Similar to FL, these exist two main challenges in federated min-max optimization: 1) datasets are generated locally at the clients and generally non-i.i.d., i.e., $D_i \neq D_j$, for $i \neq j$;
2) potentially only a subset of clients  
%(0, M]$) 
may participate in each communication round, leading to partial client participation.

In this paper, we focus on general non-convex-PL min-max problems.
Before presenting the algorithms and their convergence analysis, we first state several assumptions.
%The loss functions $f_i(\x, \y), i \in [M]$ satisfy the following two assumptions.

\begin{assum}(Lipschitz Smooth) \label{a_smooth}
    $f_i(\x, \y)$ is $L_f$-smooth, i.e., there exists a constant $L_f > 0$, so that $ \| \nabla f_i(\x_1, \y_1) - \nabla f_i(\x_2, \y_2) \|^2 \leq L_f^2 \left( \| \x_1 - \x_2 \|^2 + \| \y_1 - \y_2 \|^2 \right)$, $\forall \x_1, \x_2, \y_1, \y_2 \in \mathbb{R}^d$, $i \in [M]$.
\end{assum}

\begin{assum}(Polyak-Łojasiewicz (PL) Condition) \label{a_PL}
    There exists a constant $\mu > 0$ such that $\forall \x, \y$, $$\| \nabla_\y f(\x, \y) \|^2 \geq 2 \mu \max_{\z} \left(f(\x, \z) - f(\x, \y)\right).$$
\end{assum}
Further, we assume the stochastic gradients with respect to $\x$ and $\y$ in each local update step at each client are unbiased and have bounded variances.
\begin{assum}(Unbiased Local Stochastic Gradient) \label{a_unbias}
	Let $\xi^i$ be a random local data sample at client $i$.
	The local stochastic gradients with respect to $\x$ and $\y$ are unbiased and have bounded variances: 
	\begin{align*}
	&\mathbb{E} [\nabla f_i(\x, \y, \xi_i)] = \nabla f_i(\x, \y), \quad \mathbb{E} \left[\| \nabla_\x f_i(\x, \y, \xi_i) -  \nabla_\x f_i(\x, \y) \|^2 \right] \leq \sigma_{x}^2, \\
	&\mathbb{E} \left[\| \nabla_\y f_i(\x, \y, \xi_i) -  \nabla_\y f_i(\x, \y) \|^2 \right] \leq \sigma_{y}^2, 
	\end{align*}
	where the expectation is taken over local distribution $D_i$.
\end{assum}

To analyze the convergence performance of min-max algorithms, we define a surrogate function $\Phi$ for the global minimization as follows:
    $\Phi(\x) := \max_{\y} f(\cdot , \y).$
We will use $\Phi$ as a metric to measure the performance of an algorithm on min-max problems, and the goal is to find an approximate stationary point of $\Phi$ efficiently.
Then, we can conclude from previous works (see Lemma A.5~\cite{nouiehed2019solving} or Lemma 4.3~\cite{lin2020gradient}) that $\Phi$ is $L$-smooth, where $L := L_f + L_f^2/\mu$.

\begin{defn} 
    [Stationarity]
    For a differentiable function $\Phi$, $\z$ is an $\epsilon$-stationary point if $\| \nabla \Phi(\z) \| \leq \epsilon$.
\end{defn}

\begin{defn} [Complexity]
    The communication and client sample complexity are defined as the total number of rounds and stochastic gradients per client to achieve an $\epsilon$-stationary point, respectively.
\end{defn}

% \kliu{It may be better to give one or two concrete application examples to illustrate/motivate why the the problem in (1) and the associated assumptions are relevant.}

\subsection{The Stochastic Averaging Gradient Descent Ascent (\algnns) Algorithm} \label{subsec: sagda}

% -----------
%FSGDA algorithm
\begin{algorithm}[t!]
    \caption{The Stochastic Averaging Gradient Descent Ascent (\algnns) Algorithm.} \label{alg:sagda} 
    \begin{algorithmic}[1]
    % \STATE{Initialize $\x_0, \y_0$.}

    \FOR{$t = 0, \cdots, T-1$}
        \FOR{Server} 
        \STATE{Initialize $\x_0, \y_0$ for $t=0$, or update global model from previous round for $t > 0$: \\
        \qquad $\x_{t} = \x_{t-1} + \eta_{x, g} \left(\frac{1}{m} \sum_{i \in S_{t-1}} \x_{t-1, i}^{K+1} - \x_{t-1} \right)$, \\
        \qquad $\y_{t} = \y_{t-1} + \eta_{y, g} \left(\frac{1}{m} \sum_{i \in S_{t-1}} \y_{t-1, i}^{K+1} - \y_{t-1} \right)$.
        }
        \STATE{Randomly samples a subset $S_t$ of clients with $|S_t| = m$.}
        \STATE{{\bf Option I:} Construct sampling averaging $\bar{\v}_{x}, \bar{\v}_{y}$ from the return in the previous round: \\
        \qquad $\bar{\v}_{x} = \bar{\v}_{x} + \frac{1}{M} \sum_{i \in S_{t-1}} \Delta \v_{x, i},$ \qquad $\bar{\v}_{y} = \bar{\v}_{y} + \frac{1}{M} \sum_{i \in S_{t-1}} \Delta \v_{y, i}$.
        \\}
        \STATE{{\bf Option II:} The server sends  current parameters $\z_t := \left( \x_t, \y_t \right)$ to clients in $S_t$ and collects stochastic gradients: \\
        \qquad $\v_{x, i} = \nabla_x f_i(\z_t, \xi_{t,i}), \v_{y, i} = \nabla_y f_i(\z_t, \xi_{t,i}),$ \\
        \qquad $\bar{\v}_{x} = \frac{1}{m} \sum_{i \in S_t} \v_{x, i}, \bar{\v}_{y} = \frac{1}{m} \sum_{i \in S_t} \v_{y, i}.$}
        \STATE{Send $(\x_t, \y_t)$ and $(\bar{\v}_{x}, \bar{\v}_{y})$ to each client $i \in S_t$.}
        \ENDFOR

        \FOR{Each client $i \in S_t$}
            \STATE{Synchronization: $\x_{t, i}^1 = \x_{t}, \y_{t, i}^1 = \y_{t}$ and receiving $\bar{\v}_{x, t}, \bar{\v}_{y, t}$.}

            \STATE{Local updates ($ k \in [K]$): \\ 
            \qquad $\x_{t, i}^{k+1} = \x_{t, i}^{k} - \eta_{x, l} \v_{x, i}^{k}$ (cf. Eq.~\eqref{x_update} for $\v_{x, i}^{k}$); \\
            \qquad $\y_{t, i}^{k+1} = \y_{t, i}^{k} + \eta_{y, l} \v_{y, i}^{k}$ (cf. Eq.~\eqref{y_update} for $\v_{y, i}^{k}$);
            }
            \STATE{{\bf Option I:} \\
            \qquad Calculate: $\v_{x,i}^{'} = \nabla_x f_i(\z_t, \xi_{t, i}), \v_{y, i}^{'}  = \nabla_y f_i(\z_t, \xi_{t, i})$. \\
            \qquad Send $\left(\x_{t, i}^{K+1}, \y_{t, i}^{K+1}\right)$ and $\left(\Delta \v_{x, i}, \Delta \v_{y, i}\right) = \left(\v_{x, i}^{'} - \v_{x, i}, \v_{y, i}^{'} - \v_{y, i}\right)$ to server. \\
            \qquad Assign: $\v_{x, i} = \v_{x, i}^{'}, \v_{y, i} = \v_{y, i}^{'}$.
            }
            \STATE{{\bf Option II:} Send $\left(\x_{t, i}^{K+1}, \y_{t, i}^{K+1}\right)$ to server.}
        \ENDFOR
    \ENDFOR
    \end{algorithmic}
\end{algorithm}

To solve Problem~\eqref{eqn:minmax}, FedAvg could be naturally extended to federated min-max problems by applying SGDA with multiple local update steps in primal and dual variables respectively.
However, current results~\cite{deng2021local,xie2021federated,sharma2022federated} show that there exists two limitations: 1) limited data heterogeneity is often assumed, e.g., bounded gradient dissimilarity assumption; 2) communication complexity is unsatisfactory.
In this paper, we propose the \algn (\ul{s}tochastic \ul{s}ampling \ul{a}veraging \ul{g}radient \ul{d}escent \ul{a}scent) algorithm by utilizing the assembly of stochastic gradients from (randomly sampled) clients as control variates to mitigate the effect of data heterogeneity in federated min-max problems.
As will be shown later, \algn is able to achieve better communication complexity under arbitrary data heterogeneity.

As illustrated in Algorithm~\ref{alg:sagda}, \algn contains the following two stages:
\vspace{-.1in}
\begin{list}{\labelitemi}{\leftmargin=1.2em \itemindent=-0.0em \itemsep=.1em}
    \item[1.] {\em On the Server Side:} In each communication round, the server initializes the global model $\left(\x_t, \y_t \right)$ at $t=0$ or updates the global model accrodingly when $t > 0$ (Line 3).
    Specifically, for $t > 0$, upon the reception of all returned parameters from round $t-1$, the server aggregates them using global learning rates $\eta_{x, g}$ and $\eta_{y, g}$ for $\x$ and $\y$, respectively.
    Then server samples a subset  of clients $S_t$ to participate in the training and broadcast the current global model $(\x_t, \y_t)$ to these clients (Line 4).
    Here, we follow the same common assumption on client participation as in FL: the clients are uniformly sampled without replacement and a fixed-size subset (i.e., $| S_t | = m$) is chosen in each communication round.
    A key step here is to construct the control variates ($\bar{\v}_x, \bar{\v}_y,\v_{x, i}, \v_{y, i}$) for server and client.
%    We provide two options in Line 5 and 6.
%   {\color{red} While both options exchange the same amount of information in each round and have the same theoretical performance, Option II needs one more communication session than Option I to transmit ???, which may incur a slightly higher communication overhead.
%   Meanwhile, Option I needs to maintain the state of control variates $(\bar{\v}_x, \bar{\v}_y)$, which requires more memory.}
    Afterwards, the primal and dual variables alongside their control variates are transmitted to each participated client $i \in S_t$ (Line 7).
    \item[2.] {\em On the Client Side:} Upon receiving the latest global model $(\x_t, \y_t)$, each client synchronizes its local model (Line~10). 
    Then, each client performs $K$ local updates for $\x$ and $\y$ simultaneously (Line~11).
    Upon the completion of local computations, the new local model is sent to the server.
\end{list}

%There are two key differences in \algn compared to FSGDA.
We provide two options in \algnns.
First, in each communication round, client and server need to respectively obtain control variates ($\v_{x, i}, \v_{y, i}$) and ($\bar{\v}_{x}, \bar{\v}_{y}$) for ``variance reduction" purpose in primal variable $\x$ and dual variable $\y$ (Lines~5 and 6).
%We specify two options.
Option I requires each client to maintain the control variates ($\v_{x, i}, \v_{y, i}$) across rounds locally (Line 12).
As a result, ($\bar{\v}_{x}, \bar{\v}_{y}$) are constructed iteratively (Line 5).
In Option II, ($\v_{x, i}, \v_{y, i}$) are instantly calcuated by another round of communication, and then ($\v_{x, i}, \v_{y, i}$) are constructed accrodingly (Line 6).
We note that Option I needs client to be {\em stateful} and thus being more challenging to implement in cross-device FL~\cite{kairouz2019advances}, while Option II may incur extra communication overhead due to the need for one more communication session, although the total communication size remains the same.
In the local computation phase, each participated client performs steps (Line 11) based on Eq.~\eqref{x_update} and ~\eqref{y_update}, which can be interpreted as ``variance reduction."
Here, we use $\z_{t, i}^j := (\x_{t, i}^j, \y_{t, i}^j)$ for notational simplicity.
\begin{align}
    \v_{x, i}^{k} &= \nabla_x f_i(\z_{t, i}^k, \xi_{t,i}^k) - \v_{x, i} + \bar{\v}_{x}, \label{x_update} \\
    \v_{y, i}^{k} &= \nabla_y f_i(\z_{t, i}^k, \xi_{t,i}^k) - \v_{y, i} + \bar{\v}_{y}. \label{y_update}
\end{align}

In classic variance reduction methods, the key idea is to utilize a full gradient (or approximation) to reduce the stochastic gradient variance at the expense of high computation complexity compared to SGD.
Note that, in federated learning, the gradient dissimilarity (due to data heterogeneity) is a crtical challenge and more problematic than stochastic gradient variance.
Therefore, we calculate a 2-tuple ($\bar{\v}_{x, t}, \bar{\v}_{y, t}$) of stochastic gradients from all clients as control variates to mitigate the potential gradient deviation due to data heterogeneity.
Note that \algn does {\em not} require a full gradient calculation for each client.
With the help from the local steps in ~\eqref{x_update} and ~\eqref{y_update}, each client no longer generate large deviation in local updates even with {\em arbitrary} data heterogeneity.
The reason is that, for small local learning rates, the local steps in each client could be approximated by 
\begin{align*}
    \nabla_x f_i(\z_{t, i}^j, \xi_{t,i}^k) \approx \v_{x, i} &\Longrightarrow  \v_{x, i}^{k} \approx  \bar{\v}_{x}, \\
    \nabla_y f_i(\z_{t, i}^k, \xi_{t,i}^k) \approx  \v_{y, i} &\Longrightarrow  \v_{y, i}^{k} \approx \bar{\v}_{y}.
\end{align*}
In other words, \algn mimics mini-batch SGDA in the centralized learning by using an approximation of mini-batch stochastic gradient for the updates.
As a result, \algn is able to provide a desirable convergence rate, while allowing arbitrary data heterogeneity.
We state the convergence rate result of \algn as follows:

\begin{restatable} [Convergence Rate of \algnns] {theorem} {convergencesagda}
    \label{convergence_sagda}
    Under Assumptions~\ref{a_smooth}-~\ref{a_unbias}, define $\mathcal{L}_t = \Phi (\x_t) - \frac{1}{10} f(\x_t, \y_t)$, the output sequence $\{ \x_t \}$ generated by \algn satisfies:

   $\bullet$ For Option I with learning rates $\eta_{x,g}$, $\eta_{x,l}$, $\eta_{y,g}$, and $\eta_{y,l}$ satisfying 
    \begin{align*}
        &8K(K-1)(2K-1) L_f^2 \max \{ \eta_{x, l}^2 , \eta_{y, l}^2\} \leq 1, \\
        &\frac{1}{2} - 4 a_2 L_f^2 K^2 \left(\eta_x^2 + \eta_y^2\right) - \left(a_1 + a_2 4 L_f^2 K^2 \left(\eta_x^2 + \eta_y^2\right) \right) 160 K^2 \left(\eta_{x, l}^2 + \eta_{y, l}^2 \right) L_f^2 \geq 0, \\
        &\left[ \frac{1}{10} \eta_{x} K - 4 a_2 K^2 \eta_x^2 \right] - \left[a_1 + a_2 4 L_f^2 K^2 \left(\eta_x^2 + \eta_y^2\right) \right] 40 K^2 \eta_{x, l}^2 \geq 0, \\
        &\left[ \eta_y K \left(\frac{1}{20} - \frac{\eta_{x}}{\eta_y} \frac{L_f^2}{\mu^2}\right) - 4 a_2 K^2 \eta_y^2 \right] - \left[a_1 + a_2 4 L_f^2 K^2 \left(\eta_x^2 + \eta_y^2\right) \right] 40 K^2 \eta_{y, l}^2 \geq 0, 
    \end{align*}
    where $a_1 = K L_f^2 \left( \frac{31}{20} \eta_{x} + \frac{1}{20} \eta_y \right)$ and $a_2 = \frac{1}{2}\left(L + \frac{L_f}{10}\right) + 1 + \frac{M^2}{m^2} - \frac{M}{m}$, it holds that
    
    % Eq~\eqref{sagda1_lr1}~\eqref{sagda1_lr2}~\eqref{sagda1_lr3}~\eqref{sagda1_lr4}:
    \begin{align*}
        &\frac{1}{T} \sum_{t=0}^{T-1} \mb{E} \| \nabla \Phi (\x_t) \|^2 \leq \underbrace{\frac{2 \left(\mathcal{L}_0 - \mathcal{L}_{*} \right)}{\eta_x K T}}_{\mathrm{optimization \ error}} + \underbrace{\left[ \left(L + \frac{L_f}{10}\right) + 4\right] \frac{9}{m \eta_x} \left(\eta_x^2 \sigma_x^2 + \eta_y^2  \sigma_y^2 \right)}_{\mathrm{statistical \ error}} + \underbrace{\psi_1}_{\mathrm{local \ update \ error}}
    \end{align*}
    where $\psi_1$ is defined as follows:
    \begin{align*}
        \psi_1 = \left[L_f^2 \left( \frac{31}{20} + \frac{1}{20} \frac{\eta_y}{\eta_{x}} \right) + \left[ \frac{1}{2}\left(L + \frac{L_f}{10}\right) + 2\right] 4 L_f^2 K \left(\eta_x + \frac{\eta_y^2}{\eta_x}\right) \right] \left[ 20 K^2 \left(\eta_{x, l}^2 \sigma_x^2 + \eta_{y, l}^2 \sigma_y^2\right) \right].
    \end{align*}

    $\bullet$ For Option II with learning rates $\eta_{x,g}$, $\eta_{x,l}$, $\eta_{y,g}$, and $\eta_{y,l}$ satisfying 
    \begin{align*}
        &8K(K-1)(2K-1) L_f^2 \max \{ \eta_{x, l}^2 , \eta_{y, l}^2\} \leq 1, \\
        & \frac{1}{10} \eta_x K - \left(2\left(L + \frac{L_f}{10}\right) \eta_x^2 K^2 + 40K^2 \eta_{x, l}^2 b_1\right) \geq 0, \\
        & \eta_y K \left(\frac{1}{20} - \frac{\eta_{x}}{\eta_y} \frac{L_f^2}{\mu^2}\right) - \left(\frac{1}{5} L_f \eta_y^2 K^2 + 40K^2 \eta_{y, l}^2 b_1 \right) \geq 0,
    \end{align*}
    where $b_1 = L_f^2 \left[\frac{31}{20} \eta_{x} K + \frac{1}{20} \eta_y K + 2\left(L + \frac{L_f}{10}\right) \eta_x^2 K^2 + \frac{1}{5} L_f \eta_y^2 K^2 \right]$, it holds that
    
    % Eq~\eqref{sagda2_lr1}~\eqref{sagda2_lr2}~\eqref{sagda2_lr3}:
    \begin{align*}
        &\frac{1}{T} \sum_{t=0}^{T-1} \mb{E} \| \nabla \Phi (\x_t) \|^2 
        \leq \frac{2 \left(\mathcal{L}_0 - \mathcal{L}_{*} \right)}{\eta_x K T} + \left[\left(L + \frac{L_f}{10}\right) \frac{9 \eta_x}{m} \sigma_x^2 + + \frac{9}{10} L_f \frac{\eta_y^2}{m \eta_x}  \sigma_y^2 \right] + \psi_2.
    \end{align*}
    where $\psi_2$ is defined as follows:
    \begin{align*}
        \psi_2 = L_f^2 \left[\frac{31}{20} K + \frac{1}{20} \frac{\eta_y}{\eta_x} K + 2\left(L + \frac{L_f}{10}\right) \eta_x K^2 + \frac{1}{5} L_f \frac{\eta_y^2}{\eta_x} K^2 \right] \left[ 10 \left(16 K + 1\right) \right] \left(\eta_{x, l}^2 \sigma_x^2 + \eta_{y, l}^2 \sigma_y^2 \right).
    \end{align*}
\end{restatable}
Here $\eta_x = \eta_{x, l} \eta_{x, g}$ and $\eta_y = \eta_{y, l} \eta_{y, g}$.
The convergence rate results in Theorem~\ref{convergence_sagda} contain three terms: optimization error, statistical error and local update error.
The first two errors are similar to those in first-order stochastic methods, which are optimization errors due to initial point and statistical error originated from stochastic gradient variance.
The local updates without synchronization among clients result in deviations that contribute to the third error.
For the learning rates, if we use a sufficiently small local learning rates $\eta_{x, l}$ and $\eta_{y, l}$, it requires that $\eta_{x} K = \mathcal{O}(1)$ and $\eta_{y} K = \mathcal{O}(1)$.
% The control variates used in the updates have two benefits compared to the pure SGD updates.
% First, the sampling variance term in the convergence rate vanishes for partial client participation.
% Second, the local update error $\psi_3(\psi_4)$ further reduces and only has dependence on the stochastic gradient variances $\sigma_x^2$ and $\sigma_y^2$.
% The rationale is that each client utilizes an approximation of mini-batch stochastic gradients through these control variates to perform updates.
% The error of such an approximation is controllable and bounded even for a large local step number $K$.
% Furthermore, the gradient dissimilarity due to data heterogeneity gradually vanishes.
% As a result, with appropriately chosen learning rates, \algn is able to achieve a better convergence rate even under arbitrary data heterogeneity.

Based on Theorem~\ref{convergence_sagda}, we immediately have the following result:

\begin{restatable} [Convergence Rate of \algn] {corollary} {sagda} \label{cor: sagda}
    Let $\eta_{x} = \Theta(\frac{\sqrt{m}}{\sqrt{KT}}), \eta_{x} = \Theta(\frac{\sqrt{m}}{\sqrt{KT}})$, $\eta_{x, l} \leq \min \{\frac{1}{m^{1/2} K^{3/2}}, \frac{K^{3/4}}{m^{1/4} T^{1/4}}\}, \eta_{y, l} \leq \min \{\frac{1}{m^{1/2} K^{3/2}}, \frac{K^{3/4}}{m^{1/4} T^{1/4}}\}$, and $T = \Omega(mK)$,
    the convergence rate of \algn is $\mc{O}( \frac{1}{\sqrt{mKT}}).$
\end{restatable}

Corollary~\ref{cor: sagda} says that, for sufficiently many communication rounds ($T = \Omega(mK)$), \algn achieves the linear speedup in both $m$ and $K$.
In other words, the per-client sample complexity and communication complexity are $\mathcal{O}((1/m) \epsilon^{-4})$ and $\mathcal{O}(\epsilon^{-2})$, respectively.
The per-client sample complexity indicates the benefits of parallelism as the number of clients $m$ increases.
The communication complexity significantly improves those in existing works by at least a $(1/\epsilon)$-factor (cf. Table~\ref{tab:bound}).

\begin{algorithm}[t!]
    \caption{\ul{F}ederated \ul{S}tochastic \ul{G}radient \ul{D}escent \ul{A}scent (FSGDA) Algorithm.} \label{alg:fsgda} 
    \begin{algorithmic}[1]
    % \STATE{Initialize $\x_0, \y_0$.}

    \FOR{$t = 0, \cdots, T - 1$}
        \FOR{Server} 
            \STATE{Initialize $\left( \x_0, \y_0 \right)$ for $t=0$, or update global model from previous round for $t > 0$: 
            \begin{equation*}
            \x_{t} = \x_{t-1} + \eta_{x, g} \bigg(\frac{1}{m} \sum_{i \in S_{t-1}} \x_{t-1, i}^{K+1} - \x_{t-1} \bigg), \y_{t} = \y_{t-1} + \eta_{y, g} \bigg(\frac{1}{m} \sum_{i \in S_{t-1}} \y_{t-1, i}^{K+1} - \y_{t-1} \bigg).
            \end{equation*}
            }
            \vspace{-.1in}
            \STATE{The server randomly samples a subset $S_t$ of clients with $|S_t| = m$ and sends  current parameters $\left( \x_t, \y_t \right)$.}
        \ENDFOR
        \FOR{Each client $i \in S_t$}
            \STATE{Synchronization: $\x_{t, i}^1 = \x_{t}, \y_{t, i}^1 = \y_{t}$.}

            \STATE{Local updates ($ k \in [K]$):
            \begin{equation*} 
            \x_{t, i}^{k+1} = \x_{t, i}^{k} - \eta_{x, l} \nabla_\x f_i(\x_{t, i}^k, \y_{t, i}^k, \xi_{t, i}^k), \quad \y_{t, i}^{k+1} = \y_{t, i}^{k} + \eta_{y, l} \nabla_\y f_i(\x_{t, i}^k, \y_{t, i}^k, \xi_{t, i}^k).
            \end{equation*}
            }
            \vspace{-.1in}
            \STATE{Send $\left(\x_{t, i}^{K+1}, \y_{t, i}^{K+1}\right)$ to server.}
        \ENDFOR

    \ENDFOR
    \end{algorithmic}
\end{algorithm}

\subsection{Special case of \algnns: Federated stochastic gradient descent ascent (FSGDA)} \label{subsec: sgda}

We note that, if we set all the control variates to zero, \algn reduces to the federated stochastic gradient descent ascent (FSGDA) method, which is a natural extension of FedAvg and SGDA to federated min-max learning\footnote{Our FSGDA is in fact a generalized version of local FSGDA~\cite{deng2021local,sharma2022federated} as our FSGDA has two-sided learning rates and client sampling.
If $\eta_{x,g}=\eta_{y,g}=1$, our FSGDA is exactly the same as the standard FSGDA.
}.
% Despite \alg has been investigated in previous works~\cite{deng2021local,sharma2022federated}, 
We show that much improved convergence rate results of FSGDA can be directly implied by \algnns.

For a fair comparison with existing works, we also adopt the same bounded gradient dissimilarity assumption as in~\cite{deng2021local,sharma2022federated} to bound the second moment between gradients of local and global loss functions (i.e., quantifying data heterogeneity).
\begin{assum}(Bounded Gradient Dissimilarity) \label{a_dissimilarity}
	There exist two constants $\sigma_{x, G} \geq 0$ and $\sigma_{y, G} \geq 0$ such that $\mathbb{E} \left[\| \nabla_x f_i(\x, \y) -  \nabla_x f(\x, \y) \|^2 \right] \leq \sigma_{x, G}^2$ and $\mathbb{E} \left[\| \nabla_y f_i(\x, \y) -  \nabla_y f(\x, \y) \|^2 \right] \leq \sigma_{y, G}^2$.
\end{assum}
Assumption~\ref{a_dissimilarity} is a commonly-used assumption to quantify the data heterogeneity\cite{deng2021local,sharma2022federated}.
Based on the results and analysis of \algn, we can show the following convergence results for FSGDA:
%In such sense, \alg does not allow arbitrary data heterogeneity, and in turn shows the advances of \algn.

% \begin{enumerate}
%     \item {\em Clients Sampling:} In each communication round, the server first samples a subset  of clients $S_t$ to participate in the training and broadcast the current global model $(\x_t, \y_t)$ to these clients (Line 3).
%     Here, we follow the same common assumption on client participation as in FL: the clients are uniformly sampled without replacement and a fixed-size subset (i.e., $| S_t | = m$) is chosen in each communication round.
%     \item {\em Local Computation:} After receiving the latest global model $(\x_t, \y_t)$, the client synchronizes its local model (Line 5). 
%     Then it performs $K$ local updates steps for $\x$ and $\y$ simultaneously (Line 6). Specifically, the primal variable $\x$ is updated using the stochastic gradient descent method based on local dataset with a local learning rate $\eta_{x, l}$.
%     Meanwhile, the dual variable $\u$ is updated using the stochastic gradient ascent method with a local learning rate $\eta_{y, l}$. 
%     Upon the completion of local computations, the new local model is transmitted to the server.
%     \item {\em Global Aggregation:} Upon the reception of all returned parameters, the server aggregates them and updates the global model (Line 10) using global learning rates $\eta_{x, g}$ and $\eta_{y, g}$ for $\x$ and $\y$, respectively.
% \end{enumerate}

\begin{restatable} [Convergence Rate for FSGDA] {theorem} {convergence} \label{convergenceFSGDA}
    % Let learning rates satisfy Eq~\eqref{fsgda_lr1} ~\eqref{fsgda_lr2}~\eqref{fsgda_lr3} (in Appendix), 
    Under Assumptions~\ref{a_smooth}-~\ref{a_dissimilarity}, define $\mathcal{L}_t = \Phi (\x_t) - \frac{1}{10} f(\x_t, \y_t)$, 
    if the learning rates $\eta_{x,g}$, $\eta_{x,l}$, $\eta_{y,g}$, and $\eta_{y,l}$ satisfy:
    \begin{align*}
        &8K(K-1)(2K-1) L_f^2 \max \{ \eta_{x, l}^2 , \eta_{y, l}^2\} \leq 1, \\
        &a_1 - a_3 40 L_f^2 K^2 \eta_{x, l}^2 - \frac{\eta_y}{\eta_x} a_4 40 L_f^2 K^2 \eta_{x, l}^2 \geq 0, \\ 
        &a_2 - a_3 \frac{\eta_x}{\eta_y} 40 L_f^2 K^2 \eta_{y, l}^2 - a_4 40 L_f^2 K^2 \eta_{y, l}^2 \geq 0,
    \end{align*}
    where $a_1 = \left(\frac{1}{10} - 2(2L +\frac{1}{5} L_f) \eta_{x} K \right), a_2 = \left(\frac{1}{20} - \frac{2}{5} L_f \eta_y K - \frac{\eta_{x}}{\eta_y} \frac{L_f^2}{\mu^2}\right), a_3 = \left(\frac{31}{20} + (2L + \frac{1}{5}L_f) \eta_{x} K \right)$ and $a_4 = \left(\frac{1}{20} + \frac{1}{5} L_f \eta_y K\right)$,
    then the output sequence $\{ \x_t \}$ generated by FSGDA satisfies:
    \begin{equation*}
        \hspace{-.1in} \frac{1}{T} \sum_{t=0}^{T-1} \mb{E} \| \nabla \Phi (\x_t) \|^2 \leq \underbrace{\frac{2 \left(\mathcal{L}_0 - \mathcal{L}_T \right)}{\eta_x K T}}_{\mathrm{optimization \ error}} + \underbrace{\frac{2 \eta_{x}}{m} \left(L + \frac{L_f}{100}\right) \sigma_x^2 + \frac{ L_f \eta_y^2 }{5 m \eta_x} \sigma_y^2}_{\mathrm{statistical \ error}} + \underbrace{\psi_3}_{\substack{\mathrm{local} \\ \mathrm{update \ error}}} + \underbrace{\psi_4}_{\substack{\mathrm{sampling}\\ \mathrm{variance}}}. 
    \end{equation*}
    Here, $\psi_3$ and $\psi_4$ are defined as follows:
    \begin{align*}
        \psi_3 &= 2 \left( a_3 L_f^2 + a_4 \frac{\eta_y}{\eta_x} L_f^2 \right) \bigg[ 40K^2 \eta_{x, l}^2 \sigma_{x, G}^2 + 40K^2 \eta_{y, l}^2 \sigma_{y, G}^2 + 5K \eta_{x, l}^2 \sigma_x^2 + 5K \eta_{y, l}^2 \sigma_y^2 \bigg], \\
        \psi_4 &= \left( (2L +\frac{1}{5} L_f) \eta_{x} K \right) \left(1 - \frac{m}{M}\right) \frac{2}{m} \sigma_{x, G}^2 + \frac{2}{5m} L_f \eta_y K \frac{\eta_y}{\eta_x} \left(1 - \frac{m}{M}\right) \sigma_{y, G}^2,
    \end{align*}
    % where $\eta_x = \eta_{x, l} \eta_{x, g}$ and $\eta_y = \eta_{y, l} \eta_{y, g}$.
\end{restatable}

The convergence rate result in Theorem~\ref{convergenceFSGDA} contains four parts.
The first three terms are similar to the errors in \algn analysis.
However, one difference is that the local update error heavily depends on the data heterogeneity parameter $\sigma_{x, G}$ and $\sigma_{y, G}$.
Specifically, the error grows at least linearly with respect to the local step in terms stochastic variance $\sigma_x^2$ and $\sigma_y^2$ and quadratically in the gradient dissimilarity $\sigma_{x, G}^2$ and $\sigma_{y, G}^2$. 
This indicates that data heterogeneity is more problematic than stochastic gradient variance, yielding a larger error in local updates and thus necessitating smaller local steps.
Fortunately, this error is associated with the square of local learning rates $\eta_{x, l}^2$ and $\eta_{y, l}^2$. 
So, with sufficiently small local learning rates, $\psi_1$ can be easily reduced. 
In other words, under bounded gradient dissimilarity (i.e., data heterogeneity), small local learning rates render controllable local update error among clients.

Partial client participation by random sampling without replacement is an unbiased estimation of the global loss function and has a bounded variance, contributing to the third term $\psi_4$. 
Will full clients participation ($m = M$), this error term can be reduced to zero through our analysis.

Theorem~\ref{convergenceFSGDA} implies a {\em new} result for FSGDA: if we use sufficiently small local learning rates under full client participation, FSGDA achieves a similar convergence rate to those of SGD and SGDA:

\begin{restatable} [Convergence Rate of FSGDA under Full Client Participation] {corollary} {fsgda_1} \label{cor: FSGDA_1}
 Considering full client participation ($m = M$), let $\eta_{x} = \Theta(\frac{\sqrt{m}}{\sqrt{KT}}), \eta_{x} = \Theta(\frac{\sqrt{m}}{\sqrt{KT}})$, $\eta_{x, l} \leq \frac{1}{{(mT)}^{1/4} K^{5/4}}, \eta_{y, l} \leq \frac{1}{{(mT)}^{1/4} K^{5/4}}$,
    % Let $\eta_{x, l} = \frac{1}{\sqrt{T}K}, \eta_{x, l} = \frac{a}{\sqrt{T}K}, \eta_{x, g} = \eta_{y, g} = \sqrt{MK}$ ($a$ is a constant), 
    and $T = \Omega(mK)$, 
    the convergence rate of FSGDA algorithm is: $\mc{O}( \frac{1}{\sqrt{MKT}}).$
    %\begin{align}
        %  $\frac{1}{T} \sum_{t=0}^{T-1} \| \nabla \Phi (\x_t) \|^2 = \mc{O} ( \frac{1}{\sqrt{mKT}} + \frac{\left(1 - \frac{m}{M}\right) \sqrt{K}}{\sqrt{mT}} + \frac{1}{T} ) = \mc{O}( \frac{1}{\sqrt{MKT}} )$.
    %\end{align}
\end{restatable}
% Here, $a$ is a constant to guarantee a larger learning rate for dual variable $\y$ ($\eta_{y, l} > \eta_{x, l}$).
% This is due to the asymmetric nature of the non-convex-(strongly)- concave problem, consistent with previous results in SGDA ~\cite{lin2020gradient} and local SGDA~\cite{deng2021local}.
This convergence rate indicates the linear speedup effect in terms of both $M$ and $K$.
However, we note that this is subject to learning rates constraints, which does not allow arbitrarily many local steps. % and thus excludes one-shot learning.
Specifically, the number of local step in FSGDA is on the order of $K = \mathcal{O}(T/M)$.
Hence, FSGDA achieves per-client sample complexity $\mathcal{O}(\frac{1}{m \epsilon^4})$ and communication complexity $\mathcal{O}(\frac{1}{\epsilon^2})$.
We improve the state-of-the-art communication comoplexity from $\mathcal{O}(\frac{1}{\epsilon^3})$ in \cite{sharma2022federated} to $\mathcal{O}(\frac{1}{\epsilon^2})$ in our paper.

% In federated min-max optimization, we improve the convergence rate from $\mathcal{O}(1/(MKT)^{1/3})$ in local SGDA method~\cite{deng2021local} by a factor of $\mathcal{O}(1/(MKT)^{1/6})$.
% We note that \cite{sharma2022federated} obtains the same rate, but FSGDA enables a larger local step $K$ and thus enjoys a better communication complexity (see Table~\ref{tab:bound}).
% Specifically, we improve the communication complexity from $\mathcal{O}(\frac{1}{\epsilon^3})$ in \cite{sharma2022federated} to $\mathcal{O}(\frac{1}{\epsilon^2})$ in our paper.
% In comparison, first-order SGD-based algorithms (e.g., FedAvg) for minimization problems in FL settings achieve a tight and optimal convergence rate of $\mathcal{O}(1/(mKT)^{1/2})$~\cite{karimireddy2019scaffold,reddi2020adaptive,gu2021fast,yang2021linearspeedup}.
% It thus can be seen that Corollary~\ref{cor: FSGDA_1} matches the optimal convergence rate of federated learning for non-convex minimization problems.

For partial client participation $m < M$, however, FSGDA can only have the following convergence rate under appropriate learning rates.
\begin{restatable} [Convergence Rate of FSGDA under Partial Client Participation] {corollary} {fsgda_2} \label{cor: FSGDA_2}
    Let $\eta_{x} = \Theta(\frac{\sqrt{m}}{\sqrt{T}K}), \eta_{y} = \Theta(\frac{\sqrt{m}}{\sqrt{T}K})$, $\eta_{x, l} \leq \frac{1}{{(mT)}^{1/4} K}$, and $\eta_{y, l} \leq \frac{1}{{(mT)}^{1/4} K}$,
    % $\eta_{x, l} = \frac{1}{\sqrt{T}K}, \eta_{x, l} = \frac{a}{\sqrt{T}K}, \eta_{x, g} = \eta_{y, g} = \sqrt{m}$ ($a$ is a constant), 
    the convergence rate of FSGDA algorithm is $\mc{O}(1/\sqrt{mT}).$
\end{restatable}

For partial client participation, we note that only the linear speedup in $m$ is achievable and 
the linear speedup in $K$ is not achievable due to the impact of sampling variance.
According to previous works in federated minimization problem, the convergence bounds also have this observation~\cite{Karimireddy2020SCAFFOLD,yang2021linearspeedup}.
To our knowledge, we are not aware of any existing results on linear speedup in $K$ with partial client participation. 
We will leave this as an open problem in our future studies.

	\vspace{-.1in}
\section{Numerical Experiments} \label{sec: Experiment}

\begin{figure*}[t!]
	\centering
	\subfigure[Logistic regression under ``a9a" dataset.]{
		\includegraphics[width=0.225\textwidth]{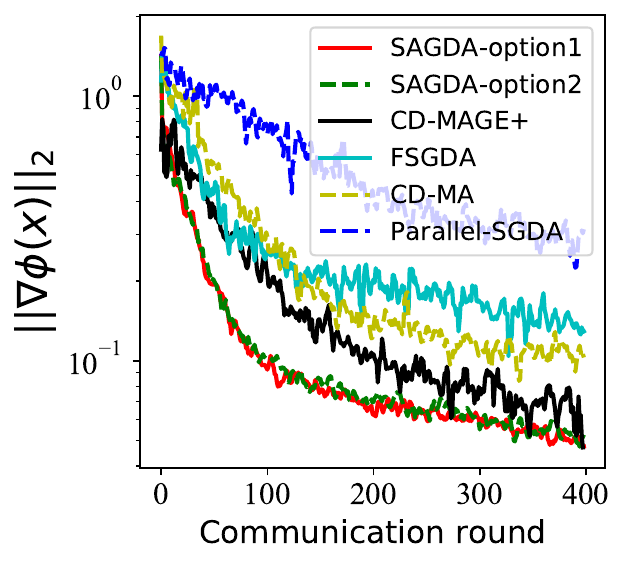}
		\label{fig:regression_grad_a9a}
	}
	\hspace{0.001\textwidth}
	\subfigure[Logistic regression under ``MNIST" dataset.]{
		\includegraphics[width=0.215\textwidth]{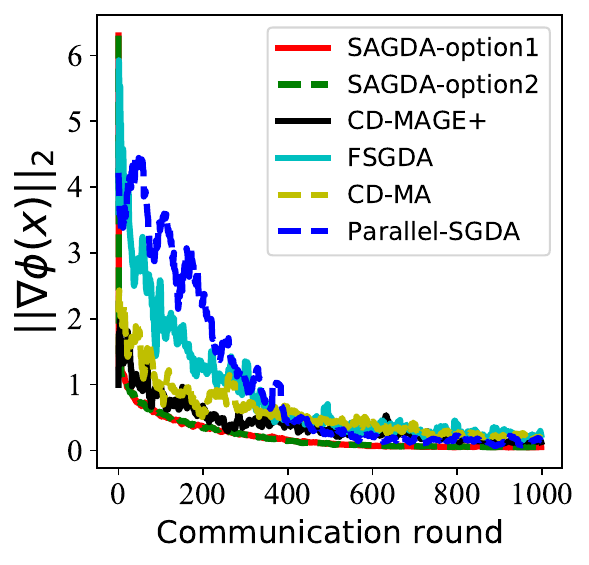}
		\label{fig:regression_grad_mnist}
	}
%	\vspace{-.03in}
	\hspace{0.001\textwidth}
	\subfigure[AUC maximization under ``a9a" dataset.]{
		\includegraphics[width=0.22\textwidth]{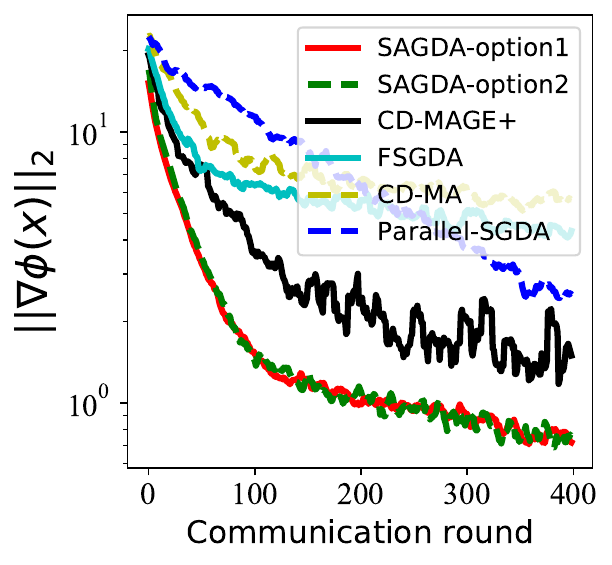}
		\label{fig:AUC_grad_a9a}
	}
	\hspace{0.001\textwidth}
	\subfigure[AUC maximization under ``MNIST" dataset.]{
		\includegraphics[width=0.21\textwidth]{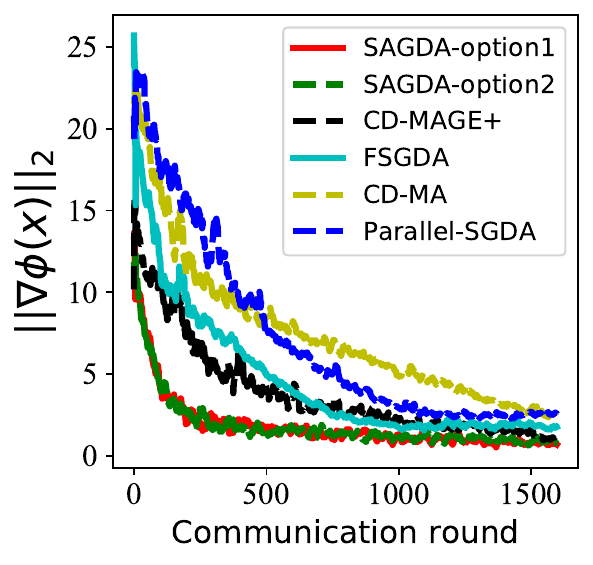}
		\label{fig:AUC_grad_mnist}
	}
	\vspace{-.1in}
	\caption{Comparisons of federated min-max learning algorithms in terms of  communication rounds.}
	\label{fig_compare}
\end{figure*}

\begin{figure*}[h!]
	\centering
	\subfigure[Logistic regression under ``a9a" dataset.]{
		\includegraphics[width=0.225\textwidth]{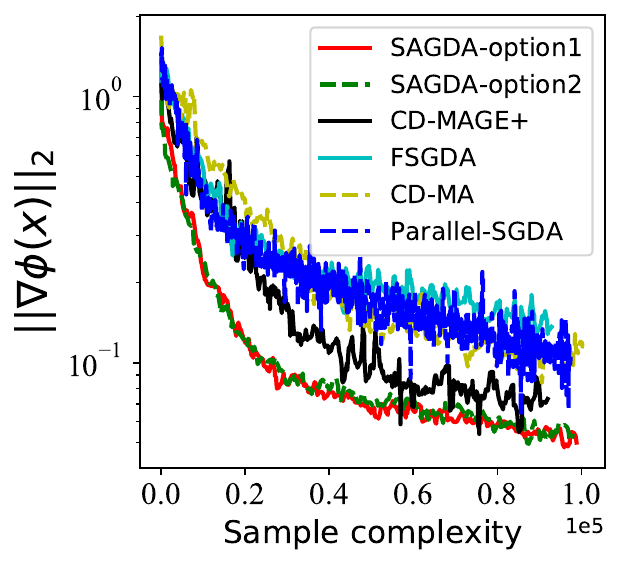}
		\label{fig:regression_grad_a9a_sample}
	}
	\hspace{0.001\textwidth}
	\subfigure[Logistic regression under ``MNIST" dataset.]{
		\includegraphics[width=0.215\textwidth]{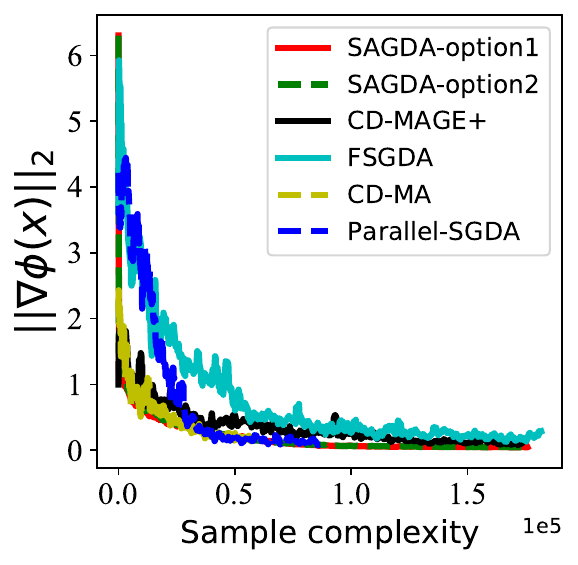}
		\label{fig:regression_grad_mnist_sample}
	}
	\hspace{0.001\textwidth}
	\subfigure[AUC maximization under ``a9a" dataset.]{
		\includegraphics[width=0.22\textwidth]{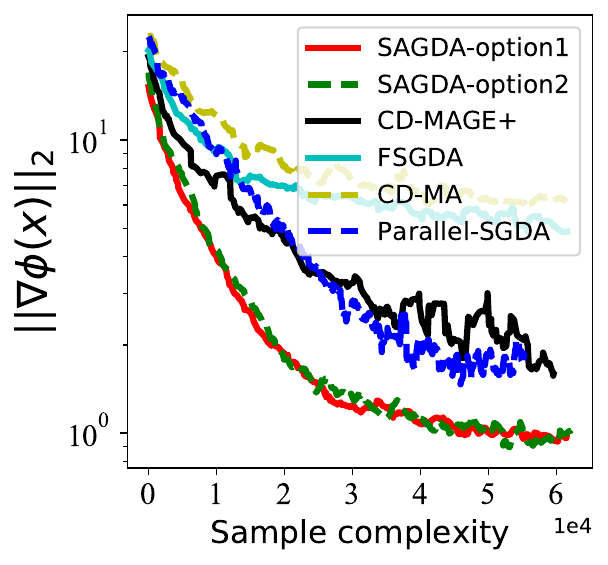}
		\label{fig:AUC_grad_a9a_sample}
	}
	\hspace{0.001\textwidth}
	\subfigure[AUC maximization under ``MNIST" dataset.]{
		\includegraphics[width=0.225\textwidth]{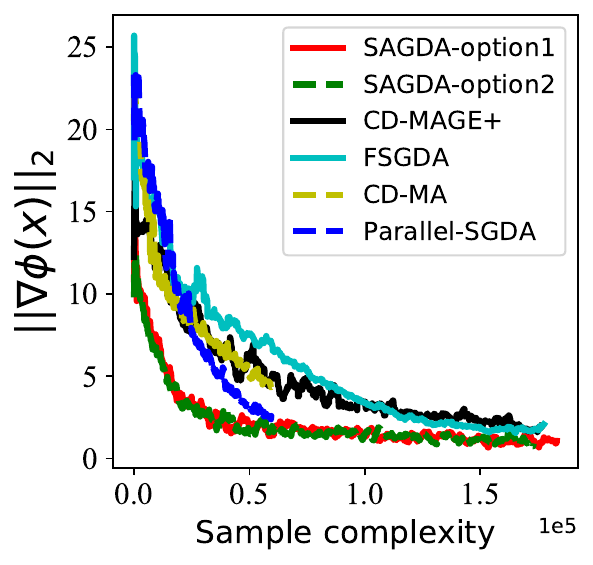}
		\label{fig:AUC_grad_mnist_sample}
	}
	%\vspace{-.1in}
	\caption{Comparisons of federated min-max learning algorithms in terms of sample complexity.}
	\label{fig_sample_compare}
\end{figure*}

In this section, we conduct numerical experiments using two machine learning problems (Logistic Regression and AUC Maximization) to verify our theoretical results for \algn as well as FSGDA. 
Due to space limitation, detailed discriptions for machine learning models and additional experiments for parameter tuning are relegated to our supplementary material.

We compare our algorithms using Parallel-SGDA~\cite{lin2020gradient,sharma2022federated,deng2021local}, CD-MA~\cite{xie2021federated}, and CD-MAGE+~\cite{xie2021federated} as baselines in our experiments. 
We note that the CD-MAGE+ method is the state-of the-art federated minimax algorithm.

\begin{list}{\labelitemi}{\leftmargin=1em \itemindent=-0.09em \itemsep=.2em}
	\item {\em Parallel-SGDA\cite{lin2020gradient,sharma2022federated,deng2021local}:} Parallel-SGDA is the parallel version of the stochastic gradient descent ascent(SGDA) algorithm. 
	Each agent $i$ updates its local parameters as:
	$\x_{t+1,i} \!=\x_{t,i}\! -\!\eta_{x,l} \nabla_{x} f_{i}(\x_{t,i},\y_{t,i},\xi_{t,i})$ and $\y_{t+1,i} \!=\y_{t,i}\! -\!\eta_{y,l} \nabla_{y} f_{i}(\x_{t,i},\y_{t,i},\xi_{t,i})$.
\item {\em CD-MA\cite{xie2021federated}:} Each agent $i$ updates its local parameters with mini-batch estimators.
The server computes $\x_{t+1} = \x_t + \frac{1}{m} \sum_{i \in S_t} \x_{t, i}^{K+1}$,
	$\y_{t+1} = \y_t + \frac{1}{m} \sum_{i \in S_t} \y_{t, i}^{K+1} $, where $S_t$ is the subsets of clients.
	\item {\em CD-MAGE+\cite{xie2021federated}:} Each agent $i$ updates its local parameters with a recursive momentum-based estimator. The server does the same procedure as in CD-MA.
\end{list}

\textbf{1) Datasets:}
We test the convergence performance of our algorithms using the ``a9a'' dataset\cite{chang2011libsvm} and ``MNIST''\cite{lecun2010mnist} from LIBSVM repository.
%
%{\color{red} 
%
The ``a9a‘’ readily contains two categories for classification.
To generate data with two categories for ``MNIST'', we split it into two classes by treating the number ``1"  class as the positive class and the remaining as the negative class.
We randomly selected 5000 data points from the positive class and 5000 data points from the negative class in the data repository.
%
% In addition, 
To generate heterogeneous data, the training data is first sorted according to the original class label and then equally partitioned into 100 workers so that all data points on one client are from the same class.
%}

%
%
%\begin{wrapfigure}{h!}{0.4\columnwidth}
%	\centering
%	\includegraphics[width=1\linewidth]{Figure/network.png}
%	\caption{Network topology.}
%	\label{fig: topos}
%\end{wrapfigure}
%

\smallskip
\textbf{2) Parameter Settings:}
We initialize all algorithms at the same point,
% which is
generated randomly from the random number generator in Python.
%The generated topology is shown in Figure~\ref{fig: topos}.
%
The learning rates are chosen as $\eta_{x,l}=\eta_{y,l}=10^{-2},\eta_{x,g}=\eta_{y,g}=2$, local updates $K=10$.
We have $m=100$ clients and each client has $n=100$ samples.

\smallskip
\textbf{3) Performance Comparisons:}
As shown in Fig.~\ref{fig_compare}, we conduct experiments by using distributionally robust optimization with non-convex regularized logistic loss and by AUC maximization on both ``a9a" and ``MNIST" datasets. 
We compare the convergence results in terms of the number of communication rounds and sample complexity. 
For better visualization, the results are smoothed by averaging the values over a window of size five.
It can be seen from Fig.~\ref{fig_sample_compare} that \algn converges faster than the baseline algorithms (CD-MA, CD-MAGE+, Parallel-SGDA) in terms of the
total number of communication rounds. 
We can also observe that \algn have a lower sample complexity than all the other algorithms.
%
%{\color{red}
As mentioned in Section.~\ref{sec: alg}, the local learning rates are necessarily small since they are used to control the local update errors. 
Thanks to the relatively large step-size $\eta_{x,g}, \eta_{y,g}$ in \algnns, the actual learning rate $\eta_x=\eta_{x,l} \times \eta_{x,g}$ and $\eta_y=\eta_{y,l} \times \eta_{y,g}$ help us achieve better convergence performance in each communication round. 
%}
%
Our experimental results thus verify our theoretical analysis that \algn is able to achieve both low sample and communication complexities.

%
%\begin{figure}[htbp!]
%	\centering
%	\subfigure[Local steps]{
%		\includegraphics[width=0.19\textwidth]{Figures/AUC_grad_K.pdf}
%		\label{fig:sample_MSPBE}
%	}
%	\hspace{0.001\textwidth}
%	\subfigure[Learning rate.]{
%		\includegraphics[width=0.2\textwidth]{Figures/AUC_grad_lr.pdf}
%		\label{fig:sample_Metric}
%	}
%	\caption{Parameter tuning of \algn.}
%	\label{fig_tuning}
%	\vspace{-.15in}
%\end{figure}
%

% !TEX root = main.tex
				\vspace{-.05in}
\section{Conclusion} \label{sec: Conclusion}
In this paper, we considered federated min-max learning with the goal of achieving low communication complexity.
We proposed a new algorithmic framework called \algnns, which i) assembles stochastic gradient estimators from randomly sampled clients as control variates  and ii) leverages two learning rates on both server and client sides.
We showed that \algn achieves a linear speedup in terms of both the number of clients and local update steps, which yields an $\mathcal{O}(\epsilon^{-2})$ communication complexity that is orders of magnitude lower than the state of the art.
Also, by noting that the standard federated stochastic gradient descent ascent (FSGDA) is in fact a special case of \algnns, we also obtained an $\mathcal{O}(\epsilon^{-2})$ communication complexity result for FSGDA.
Extensive numerical experiments corroborated the effectiveness and efficiency of our algorithms.

\section*{Acknowledgements}
This work has been supported in part by NSF grants CAREER CNS-2110259, CNS-2112471, ECCS-2140277, and CCF-2110252.

\bibliographystyle{IEEEtran}{}
\bibliography{BIB/FederatedLearning, BIB/StatisticalLearning, BIB/Teams, BIB/Minimax}

% !TEX root = ../main.tex

%%%%%%%%%%%%%%%%%%%%%%%%%%%%%%%%%%%%%%%%%%%%%%%%%%%%%%%%%%%%
\section*{Checklist}

%%% BEGIN INSTRUCTIONS %%%
% The checklist follows the references.  Please
% read the checklist guidelines carefully for information on how to answer these
% questions.  For each question, change the default \answerTODO{} to \answerYes{},
% \answerNo{}, or \answerNA{}.  You are strongly encouraged to include a {\bf
% 	justification to your answer}, either by referencing the appropriate section of
% your paper or providing a brief inline description.  For example:
% \begin{itemize}
% 	\item Did you include the license to the code and datasets? \answerYes{See Section~\ref{gen_inst}.}
% 	\item Did you include the license to the code and datasets? \answerNo{The code and the data are proprietary.}
% 	\item Did you include the license to the code and datasets? \answerNA{}
% \end{itemize}
% Please do not modify the questions and only use the provided macros for your
% answers.  Note that the Checklist section does not count towards the page
% limit.  In your paper, please delete this instructions block and only keep the
% Checklist section heading above along with the questions/answers below.
%%% END INSTRUCTIONS %%%

\begin{enumerate}

	\item For all authors...
	\begin{enumerate}
		\item Do the main claims made in the abstract and introduction accurately reflect the paper's contributions and scope?
		\answerYes{See Section~\ref{sec: Introduction}.}
		\item Did you describe the limitations of your work?
		\answerYes{}
		We need a set of control variates to transmit in each round. 
		\item Did you discuss any potential negative societal impacts of your work?
		\answerNo{As a theoretical paper towards further understanding of federated min-max optimization, we do not see a direct path to any negative applications.}
		\item Have you read the ethics review guidelines and ensured that your paper conforms to them?
		\answerYes{}
	\end{enumerate}

	\item If you are including theoretical results...
	\begin{enumerate}
		\item Did you state the full set of assumptions of all theoretical results?
	\answerYes{See Section~\ref{sec: alg}.}
		\item Did you include complete proofs of all theoretical results?
	\answerYes{See Section~\ref{sec: alg}.}
	\end{enumerate}

	\item If you ran experiments...
	\begin{enumerate}
		\item Did you include the code, data, and instructions needed to reproduce the main experimental results (either in the supplemental material or as a URL)?
	    \answerYes{See detailed instructions in Section~\ref{sec: Experiment} and Appendix, but codes are private.}
		\item Did you specify all the training details (e.g., data splits, hyperparameters, how they were chosen)?
		\answerYes{See Section~\ref{sec: Experiment} and Appendix.}
		\item Did you report error bars (e.g., with respect to the random seed after running experiments multiple times)?
		\answerNo{}
		\item Did you include the total amount of compute and the type of resources used (e.g., type of GPUs, internal cluster, or cloud provider)?
		\answerNo{}
	\end{enumerate}

	\item If you are using existing assets (e.g., code, data, models) or curating/releasing new assets...
	\begin{enumerate}
		\item If your work uses existing assets, did you cite the creators?
	  \answerYes{}
		\item Did you mention the license of the assets?
		\answerNA{}
		\item Did you include any new assets either in the supplemental material or as a URL?
		\answerNo{}
		\item Did you discuss whether and how consent was obtained from people whose data you're using/curating?
		\answerNo{}
		\item Did you discuss whether the data you are using/curating contains personally identifiable information or offensive content?
		\answerNo{}
	\end{enumerate}

	\item If you used crowdsourcing or conducted research with human subjects...
	\begin{enumerate}
		\item Did you include the full text of instructions given to participants and screenshots, if applicable?
		\answerNA{}
		\item Did you describe any potential participant risks, with links to Institutional Review Board (IRB) approvals, if applicable?
		\answerNA{}
		\item Did you include the estimated hourly wage paid to participants and the total amount spent on participant compensation?
		\answerNA{}
	\end{enumerate}

\end{enumerate}

\newpage
\appendix
% !TEX root = main.tex
\section{ Further Experiments and Additional Results}

In the following, we provide the detailed machine learning models for our experiments:

\smallskip
\textbf{1) Logistic Regression Model:}
We use the following min-max regression problem with datasets
$\xi_i:=\left\{\left(\mathbf{a}_{ij}, b_{ij}\right)\right\}_{j=1}^{n}$, where $\mathbf{a}_{ij} \in \mathbb{R}^{d}$ is the feature of the $j$-th sample of worker $i$ and $b_{ij} \!\in\! \{1,-1\}$ is the associated label:
\begin{align*}
	&\min _{{\x} \in \mb{R}^d}\max _{{\y} \in \mb{R}^n} \frac{1}{m}\sum_{i\in M} f_i({\x}, {\y}),
\end{align*}
where $f_i({\x}, {\y})$ is defined as:
\begin{align} 
	f_i({\x}, {\y}) \triangleq &   \frac{1}{n}\sum_{j=1}^n [{y}_{j} l_{j}({\x})-V({\y})+g({\x})], \!\!\!
\end{align}
where the loss function $l_{i}({\x}) \triangleq \log \left(1+\exp \left(-b_{ij} \mathbf{a}_{ij}^{\top} {\x} \right)\right)$
$, g({\x}) \triangleq \lambda_{2} \sum_{k=1}^{d} \frac{\alpha {x}_{k}^{2}}{1+\alpha {x}_{k}^{2}}
$, and $V(\boldsymbol{\y})=\frac{1}{2} \lambda_{1}\|n {\y}-\mathbf{1}\|_{2}^{2}$.
We choose constants $\lambda_{1}=1 / n^{2}$, $ \lambda_{2}=10^{-3}$ and $\alpha=10$. % for experiments.
%'

\textbf{2) AUC Maximization:}
We use a dataset $\left\{\mathbf{a}_{ij}, b_{ij}\right\}_{j=1}^{n}$, where $\mathbf{a}_{ij} \in \mathbb{R}^{d}$ is the feature of the $j$-th sample of worker $i$, $\mathbf{w}_{i}$ denotes a feature vector and $b_{ij} \in\{-1,+1\}$ denotes the corresponding label. For a scoring function $h_{\boldsymbol{\x}}$ of a classification model parameterized by $\boldsymbol{\x} \in \mathbb{R}^{d}$, the AUC maximization problem is defined as:
\begin{align} \label{eqn_1}
	\max _{\boldsymbol{\x}} \frac{1}{m^{+} m^{-}} \sum_{b_{ij}=+1, b_{ik}=-1} \mathbb{I}_{\left\{h_{\boldsymbol{\x}}\left(\mathbf{a}_{ij}\right) \geq h_{\boldsymbol{\x}}\left(\mathbf{a}_{ik}\right)\right\}},
\end{align}
where $m^{+}$ denotes the number of positive samples, $m^{-}$ denotes the number of negative samples, and $\mathbb{I}_{\{\cdot\}}$ represents the indicator function.
The above optimization problem can be reformulated as the following min-max optimization problem \cite{liu2019stochasticAUC,ying2016stochasticAUC}:
\begin{align}
	&	\min _{(\boldsymbol{\x}, c_1, c_2) \in \mathbb{R}^{d+2}} \max _{\lambda \in \mathbb{R}} f(\boldsymbol{\x}, c_1, c_2, \lambda)\notag\\
	:=&\frac{1}{mn}\!\sum_{i\in M} \! \sum_{j=1}^{n}\!\big\{\!(1\!-\!\tau)\!\left(h_{\boldsymbol{\x}}\!\left(\mathbf{a}_{ij}\!\right)\!-\!c_1\right)^{2} \!\mathbb{I}_{\left\{b_{ij} =1\right\}}\!-\!\tau(1\!-\!\tau) \!\lambda^{2}\notag\\
	&\!	+\!	\tau\!	\left(	h_{\boldsymbol{\x}}\left(\mathbf{a}_{ij}\right)\!	-\!	c_2\right)^{2} \mathbb{I}_{\left\{b_{ij}=-1\right\}}
	\!+\!2(1\!+\!\lambda) \tau h_{\boldsymbol{\x}}\left(\mathbf{a}_{ij}\right) \mathbb{I}_{\left\{b_{ij}=-1\right\}}
	\notag\\&
	-2(1+\lambda)(1-\tau) h_{\boldsymbol{\x}}\left(\mathbf{a}_{ij}\right) \mathbb{I}_{\left\{b_{ij}=1\right\}},
	\big\},
\end{align}

where $\tau:=m^{+} /\left(m^{+}+m^{-}\right)$ is the fraction of positive data. 
Note that $f(\boldsymbol{\x}, c_1, c_2, \cdot)$ is strongly concave for any $(\boldsymbol{\x}, c_1, c_2) \in \mathbb{R}^{d+2}$.

\begin{wrapfigure}{R}{0.3\columnwidth}
	\centering
	\includegraphics[width=1.0\linewidth]{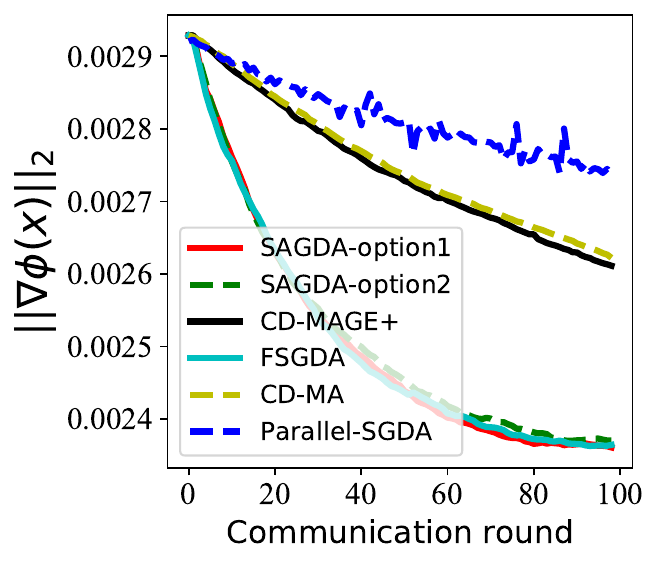}
	\caption{GANs under ``MNIST" dataset.}
	\label{fig_GAN}
\end{wrapfigure}

\textbf{3) Generator Adverserial Networks(GANs):}
Although our paper is focused on general non-convex-PL min-max problems, we believe that our paper will benefit from comparing further experimental results on the convergence performance of nonconvex-nonconcave problems (e.g., GANs), since the non-convex-PL problem is a special case for nonconvex-nonconcave min-max problems. 

In our experiment, generator network is parameterized by $\mathbf{x}$ as $G_{\mathbf{X}}$ and the discriminator network parameterized by $\mathbf{y}$ as $D_{\mathbf{y}}$. We adopt the following loss function:
$$
f_{i}(\mathbf{x}, \mathbf{y})= \mathbb{E}_{\mathbf{a}_{i}\sim \mathcal{P}_{true}} [ \log D_{\mathbf{y}}\left(\mathbf{a}_{i}\right)]+\mathbb{E}_{\mathbf{z} \sim \mathcal{P}_{\mathbf{z}}}\left[\log \left(1-D_{\mathbf{y}}\left(G_{\mathbf{x}}(\mathbf{z})\right)\right)\right]
$$
where $\mathbf{a}_{i}$ is the data point on client $i$ and $\mathcal{P}_{true}$ is the  distribution of the true samples.
$z$ denotes the input noise vector and $\mathcal{P}_{z}$ is the prior distribution of the noise vector for generating samples.
We have tested the convergence performance of our algorithms using the MNIST dataset.
We chose the learning rates as $\eta_{x,l}=\eta_{y,l}=10^{-2},\eta_{x,g}=\eta_{y,g}=2$, local updates $K=10$.
We have $m=100$ clients and each client has $n=100$ samples.
Again, from Fig.~\ref{fig_GAN}, we can observe that both our proposed algorithms FSGDA and \alg have better convergence performance compared with the baselines.

 \textbf{Impact of the Local Steps:} 
In this section, we run additional experiments to investigate the impact of the local steps $K$ on the training performance.
We run FSGDA and \algn over the hetergenous ``a9a''\cite{chang2011libsvm} dataset with the regression model mentioned in Section \ref{sec: Experiment}.
We fix the local step-size at $0.01$, worker number at $100$, and choose the number of local update rounds $K$ from the discrete set $\{2, 10,20\}$.
In terms of communication round, the gradient norm $\|\nabla\phi(\x)\|^2$ decreases as $K$ increases. This is due to the fact that the algorithm needs more communication round while $K$ is small, which matches our Corollary \ref{cor: FSGDA_1} and Corollary \ref{cor: FSGDA_2}.

\begin{figure}[htbp]
	\subfigure[ Algorithm \algnns.]{
		\begin{minipage}[t]{0.24\linewidth}
			\centering
			\includegraphics[scale=0.35]{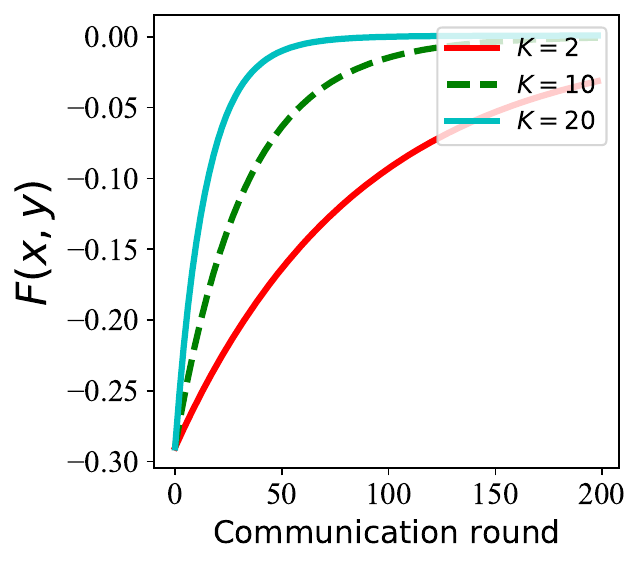} 
		\end{minipage}%
		\begin{minipage}[t]{0.24\linewidth}
			\centering
			\includegraphics[scale=0.35]{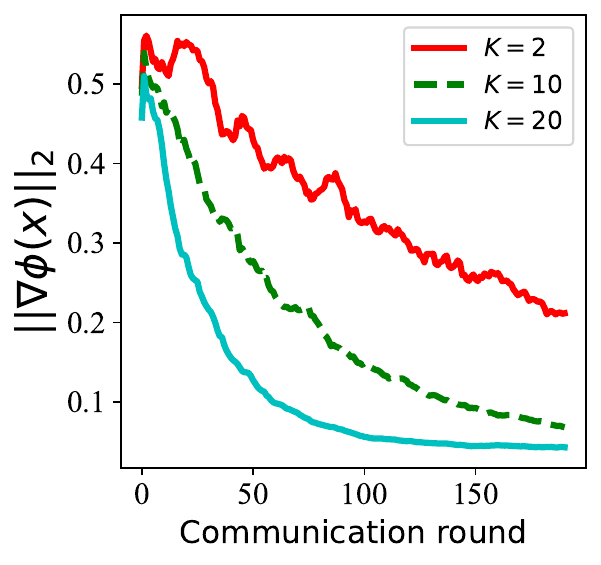}			
			%\caption{fig2}
		\end{minipage}%
		\label{img2aa}
		%	\caption{Distributed SA-SpiderBoost stability}
		%	}%
	}	     
	\subfigure[ Algorithm FSGDA.]{
		%	\subfigure[pic3.]{
		\begin{minipage}[t]{0.24\linewidth}
			\centering
			\includegraphics[scale=0.35]{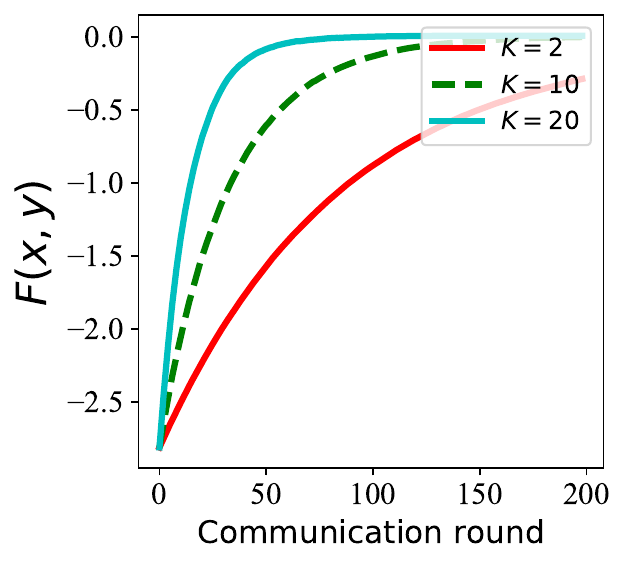}
			%\caption{fig2}
		\end{minipage}
		%	}%
		%	\subfigure[Comparison of three methods]{
		\begin{minipage}[t]{0.24\linewidth}
			\centering
			\includegraphics[scale=0.35]{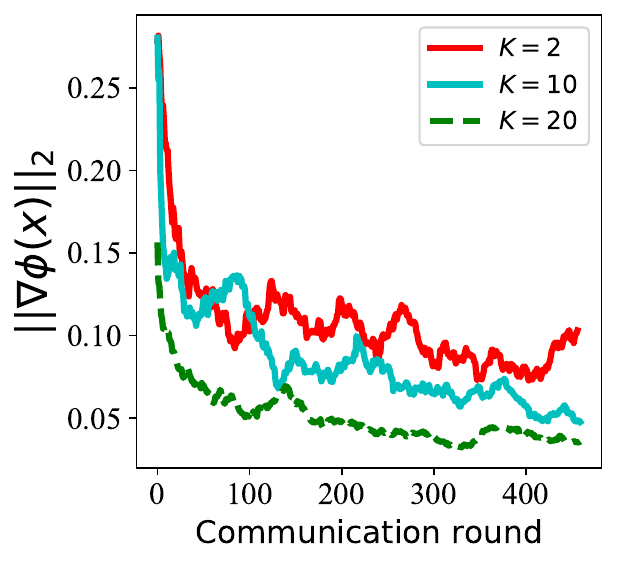}
			%\caption{fig2}
		\end{minipage}
		%	}%
		\centering
		\label{img2bb}
	}
	%	\caption{Distributed SA-SpiderBoost stability}
	%	}%     
	\caption{Algorithm performance under different local $K$ steps.}
	\label{fig:compare_K} 
\end{figure}

 \textbf{Impact of the Local Step-size:} 
 In this experiment, we choose the value of the local step-sizes from the discrete set $\{0.0001, 0.001, 0.01\}$ and fix worker number at $100$, local update rounds at $10$.
 As shown in Fig.~\ref{fig:FSGDA_step_sizea} and Fig.\ref{fig:FASGDA_step_sizea} , larger local step-sizes lead to faster convergence rates.
 
 \textbf{Impact of the Global Step-size:} 
 we choose the global step-sizes value from the discrete set $\{2, 5, 10\}$ and fix worker number at $100$, local update rounds at $10$.
 As shown in Fig.~\ref{fig:FSGDA_step_sizeb} and \ref{fig:FASGDA_step_sizeb} and, larger global step-sizes lead to faster convergence rates.

 \begin{figure}[htbp]
 	\subfigure[ Different local step-sizes.]{
 		\begin{minipage}[t]{0.24\linewidth}
 			\centering
 			\includegraphics[scale=0.35]{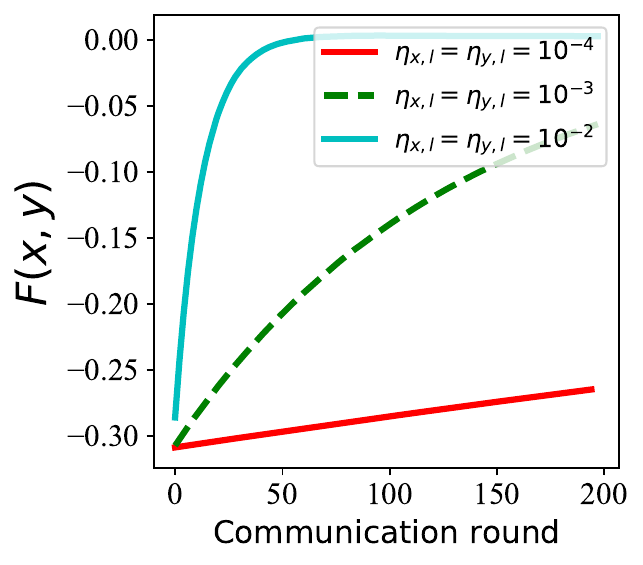} 
 		\end{minipage}%
 		\begin{minipage}[t]{0.24\linewidth}
 			\centering
 			\includegraphics[scale=0.35]{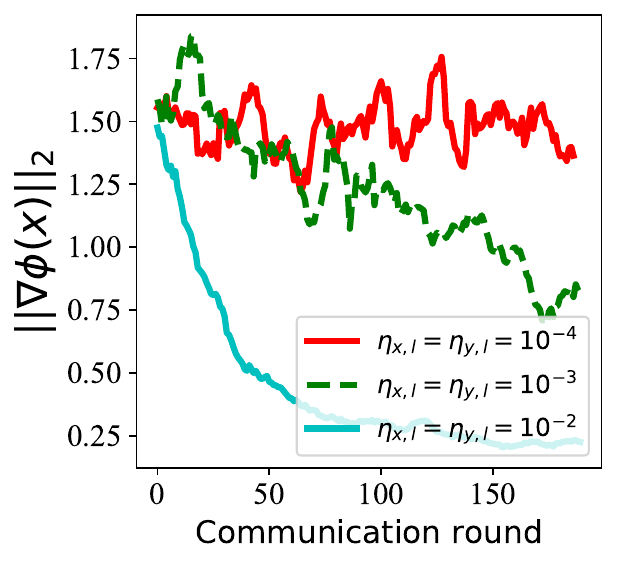}			
 			%\caption{fig2}
 		\end{minipage}%
 	\label{fig:FSGDA_step_sizea} 
 		%	\caption{Distributed SA-SpiderBoost stability}
 		%	}%
 	}	     
 	\subfigure[ Different global step-sizes.]{
 		%	\subfigure[pic3.]{
 		\begin{minipage}[t]{0.24\linewidth}
 			\centering
 			\includegraphics[scale=0.35]{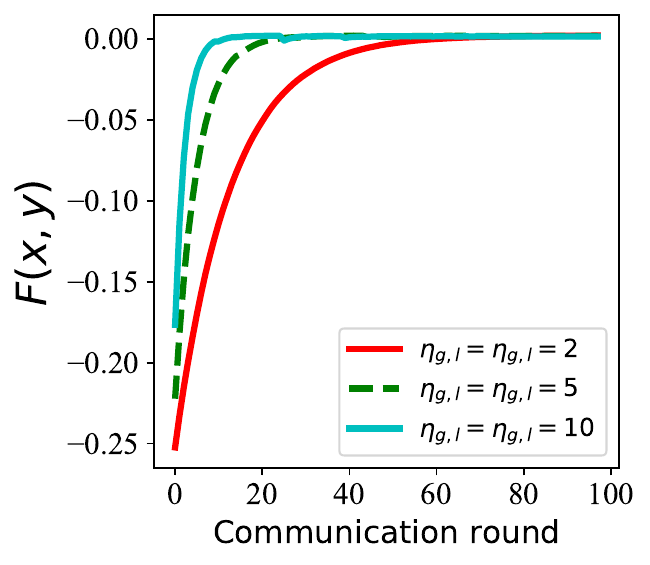}
 			%\caption{fig2}
 		\end{minipage}
 		%	}%
 		%	\subfigure[Comparison of three methods]{
 		\begin{minipage}[t]{0.24\linewidth}
 			\centering
 			\includegraphics[scale=0.35]{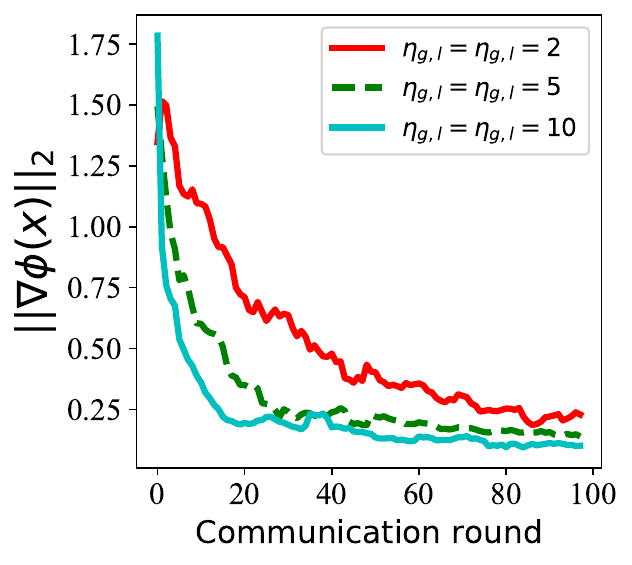}
 			%\caption{fig2}
 		\end{minipage}
 		%	}%
 		\centering
 	\label{fig:FSGDA_step_sizeb} 
 	}
 	%	\caption{Distributed SA-SpiderBoost stability}
 	%	}%     
 	\caption{The FSGDA algorithm under different step-sizes.}
  	\label{fig:FSGDA_step_size} 
 \end{figure}

 \begin{figure}[htbp]
 	\subfigure[ Different local step-sizes.]{
 		\begin{minipage}[t]{0.24\linewidth}
 			\centering
 			\includegraphics[scale=0.35]{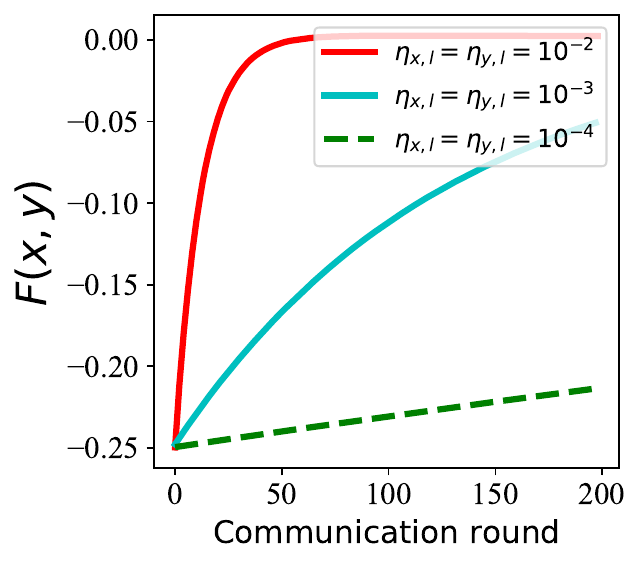} 
 		\end{minipage}%
 		\begin{minipage}[t]{0.24\linewidth}
 			\centering
 			\includegraphics[scale=0.35]{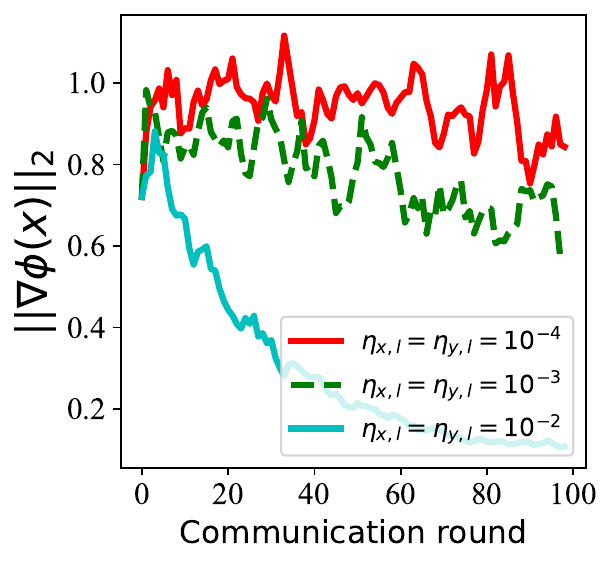}			
 			%\caption{fig2}
 		\end{minipage}%
\label{fig:FASGDA_step_sizea} 
 		%	\caption{Distributed SA-SpiderBoost stability}
 		%	}%
 	}	     
 	\subfigure[ Different global step-sizes.]{
 		%	\subfigure[pic3.]{
 		\begin{minipage}[t]{0.24\linewidth}
 			\centering
 			\includegraphics[scale=0.35]{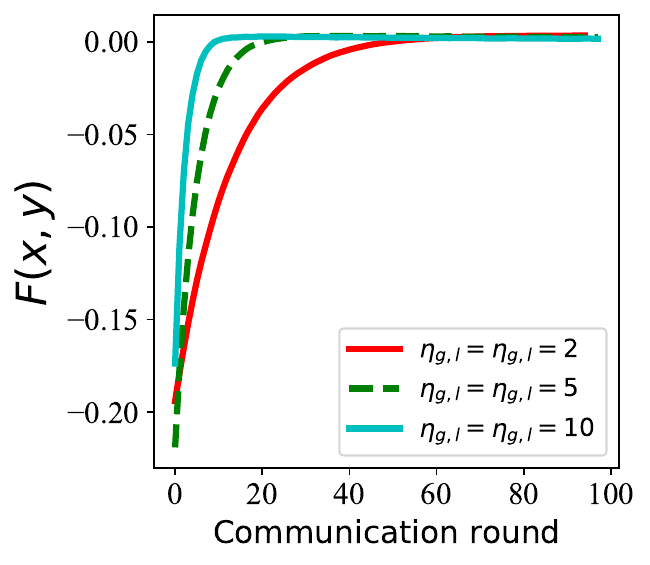}
 			%\caption{fig2}
 		\end{minipage}
 		%	}%
 		%	\subfigure[Comparison of three methods]{
 		\begin{minipage}[t]{0.24\linewidth}
 			\centering
 			\includegraphics[scale=0.35]{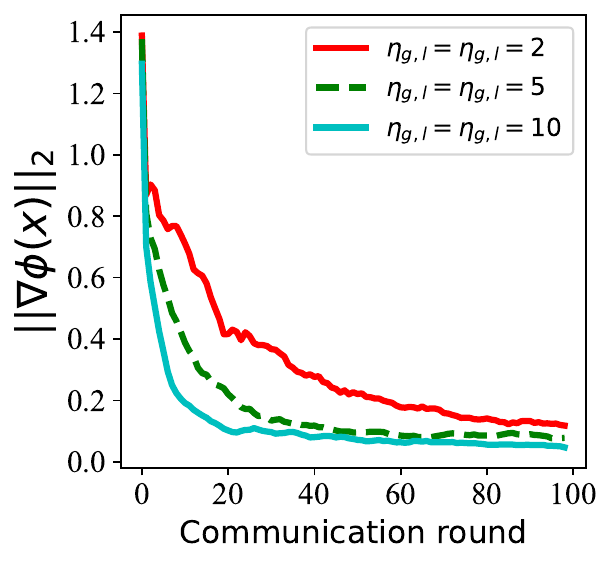}
 			%\caption{fig2}
 		\end{minipage}
 		%	}%
 		\centering
\label{fig:FASGDA_step_sizeb} 
 	}
 	%	\caption{Distributed SA-SpiderBoost stability}
 	%	}%     
  	\caption{The \algn algorithm under different step-sizes.}
 	\label{fig:FASGDA_step_size} 
 \end{figure}

%

% !TEX root = main.tex

\section{Proof} \label{sec: Proof}
\allowdisplaybreaks

\subsection{Proof for \alg } \label{subsec: fedavg}
For notational simplicity and clarity, we have the following definitions.
\begin{align*}
    &\Phi (\x) = \max_{\y \in \mb{R}^d} f(\x, \y); \\
    &\z_t = \left( \x_t, \y_t \right); \\
    &\eta_{x} = \eta_{x, g} \eta_{x, l}, \eta_{y} = \eta_{y, g} \eta_{y, l}; \\
    &\u_{x, t} = \frac{1}{m} \sum_{i \in S_t} \nabla_x f_i(\z_t), \u_{y, t} = \frac{1}{m} \sum_{i \in S_t} \nabla_y f_i(\z_t).
\end{align*}

For simplicity, we write the update step uniformly:
\begin{align*}
    \x_{t+1} 
    &= \x_t - \eta_{x} K (\u_{x, t} - \e_{x, t}), \\
    \y_{t+1}
    &= \y_t + \eta_{y} K (\u_{y, t} - \e_{y, t}).
\end{align*}

For \alg, the update rule is:
\begin{align*}
    \x_{t+1} &= \x_t - \eta_{x, g} \eta_{x, l} \left(\frac{1}{m} \sum_{i \in S_t} \sum_{j \in [K]} \nabla_x f_i(\z_{t, i}^j, \xi_{t,i}^j) \right), \\
    \y_{t+1} &= \y_t + \eta_{y, g} \eta_{y, l} \left(\frac{1}{m} \sum_{i \in S_t} \sum_{j \in [K]} \nabla_y f_i(\z_{t, i}^j, \xi_{t,i}^j) \right), \\
    \e_{x, t} &= \frac{1}{mK} \sum_{i \in S_t} \sum_{j \in [K]} \left( \nabla_x f_i(\z_t) - \nabla_x f_i(\z_{t, i}^j, \xi_{t,i}^j) \right), \\
    \bar{\e}_{x, t} &= \mb{E} [\e_{x, t}] = \frac{1}{mK} \sum_{i \in S_t} \sum_{j \in [K]} \left( \nabla_x f_i(\z_t) - \nabla_x f_i(\z_{t, i}^j) \right), \\
    \e_{y, t} &= \frac{1}{mK} \sum_{i \in S_t} \sum_{j \in [K]} \left( \nabla_y f_i(\z_t) - \nabla_y f_i(\z_{t, i}^j, \xi_{t,i}^j) \right), \\
    \bar{\e}_{y, t} &= \mb{E} [\e_{y, t}] = \frac{1}{mK} \sum_{i \in S_t} \sum_{j \in [K]} \left( \nabla_y f_i(\z_t) - \nabla_y f_i(\z_{t, i}^j) \right).
\end{align*}
Note the above expectation is only on the stochastic noise.

\begin{restatable} {lemma} {stochasticgradient}
    \label{lemma:sg}
        \begin{align*}
            \mb{E} \left\| \Delta \x_t \right\|^2 &= \mb{E}\| \left(\u_{x, t} - \e_{x, t} \right) \|^2 \leq 4 \mb{E} \left\| \bar{\e}_{x, t} \right\|^2 + 4 \mb{E} \left\| \u_{x, t} \right\|^2 + \frac{2}{mK} \sigma_x^2, \\
            \mb{E} \left\| \Delta \x_t \right\|^2 &= \mb{E}\| \left(\u_{y, t} - \e_{y, t} \right) \|^2 \leq 4 \mb{E} \left\| \bar{\e}_{y, t} \right\|^2 + 4 \mb{E} \left\| \u_{y, t} \right\|^2 + \frac{2}{mK} \sigma_y^2.
        \end{align*}
\end{restatable}
\begin{proof}
    \begin{align*}
        \mb{E} \|\left(\u_{x, t} - \e_{x, t} \right) \|^2 
        &= \mb{E} \|\left(\u_{x, t} - \bar{\e}_{x, t} \right) + \left(\bar{\e}_{x, t} - \e_{x, t} \right) \|^2 \\
        &\leq 2 \mb{E} \|\left(\u_{x, t} - \bar{\e}_{x, t} \right) \|^2 + 2 \mb{E} \| \left(\bar{\e}_{x, t} - \e_{x, t} \right) \|^2 \\
        &\leq 4 \mb{E}\| \bar{\e}_{x,t} \|^2 + 4 \mb{E}\| \u_{x,t} \|^2 + \frac{2}{mK} \sigma_x^2,
    \end{align*}
    where the second inequality follows from the fact that $\{ \nabla_x f_i(\z_{t, i}^j, \xi_{t,i}^j) - \nabla_x f_i(\z_{t, i}^j) \}$ the martingale difference sequence (see Lemma 4 in \cite{Karimireddy2020SCAFFOLD}).
    % and for variance reduction $\v_{t, i}^j = \nabla_x f_i(\z_{t, i}^j, \xi_{t,i}^j) - \nabla_x f_i(\z_{t, i}^{j-1}, \xi_{t,i}^{j-1}) + \nabla_x f_i(\z_{t, i}^1, \xi_{t,i}^1), \bar{\v}_{t, i}^j = \nabla_x f_i(\z_{t, i}^j) - \nabla_x f_i(\z_{t, i}^{j-1}) + \nabla_x f_i(\z_{t, i}^1)$.

    The bound of $\| \left(\u_{y, t} - \e_{y, t} \right) \|^2$ follows from the similar proof.
\end{proof}

\begin{restatable} [One Round Progress for $\Phi$] {lemma} {phi}
    \label{lemma:phi}
    \begin{align*}
        \mb{E} \Phi (\x_{t+1}) - \Phi (\x_t)
        &\leq - \frac{1}{2}\eta_{x} K \| \nabla \Phi (\x_t) \|^2 - \frac{1}{4} \eta_{x} K \| \nabla_x f(\z_t) \|^2 + 2L \eta_{x}^2 K^2 \mb{E}\left\| \u_{x, t} \right\|^2 \\
        &\qquad + \eta_{x} K \left(\frac{3}{2} + 2L \eta_{x} K \right) \mb{E}\left\| \bar{\e}_{x, t} \right\|^2 + \eta_{x} K \frac{L_f^2}{\mu^2} \left\| \nabla_y f(\z_t) \right\|^2  + \frac{L \eta_{x}^2 K}{m} \sigma_x^2.
    \end{align*}
\end{restatable}

\begin{proof}
Due to the $L$-smoothness of $\Phi (\x)$, we have one step update in expectation conditioned on step $t$:
\begin{align*}
    \mb{E} \Phi (\x_{t+1}) - \Phi (\x_t) &\leq \left< \nabla \Phi (\x_t), \mb{E} [\x_{t+1} - \x_t] \right> + \frac{L}{2} \mb{E} \| \x_{t+1} - \x_t \|^2 \\
    &= \underbrace{\left< \nabla \Phi (\x_t), - \eta_{x} K \mb{E} \left[\u_{x, t} - \e_{x, t} \right] \right>}_{A_1} + \underbrace{\frac{L}{2} \mb{E} \| \eta_{x} K \left(\u_{x, t} - \e_{x, t} \right) \|^2}_{A_2}.
\end{align*}

\begin{align*}
    &A_1 = \left< \nabla \Phi (\x_t), - \eta_{x} K \mb{E} \left( \nabla_x f(\z_t) - \bar{\e}_{x, t} \right) \right> \\
    &= - \frac{1}{2}\eta_{x} K \| \nabla \Phi (\x_t) \|^2 - \frac{1}{2}\eta_{x} K \mb{E} \| \nabla_x f(\z_t) - \bar{\e}_{x, t} \|^2 + \frac{1}{2} \eta_{x} K \mb{E} \| \nabla \Phi (\x_t) - \nabla_x f(\z_t) + \bar{\e}_{x, t} \|^2 \\
    &\leq - \frac{1}{2}\eta_{x} K \| \nabla \Phi (\x_t) \|^2 - \frac{1}{4}\eta_{x} K \| \nabla_x f(\z_t) \|^2 + \frac{3}{2}\eta_{x} K \mb{E} \| \bar{\e}_{x, t} \|^2 + \eta_{x} K \| \nabla \Phi (\x_t) - \nabla_x f(\z_t) \|^2,
\end{align*}
where the last inequality follows from $\| \a + \b \|^2 \geq \frac{1}{2} \| \a \|^2 - \| \b \|^2$ and $\| \a + \b \|^2 \leq 2 \| \a \|^2 + 2 \| \b \|^2 $.

\begin{align*}
    A_2 &\leq 2L \eta_{x}^2 K^2 \mb{E}\left\| \bar{\e}_{x, t} \right\|^2 + 2L \eta_{x}^2 K^2 \mb{E}\left\| \u_{x, t} \right\|^2 + \frac{L \eta_{x}^2 K}{m} \sigma_x^2,
\end{align*}
where the inequality is due to Lemma~\ref{lemma:sg}.

\begin{align*}
\| \nabla \Phi (\x_t) - \nabla_x f(\z_t) \|^2 &= L_f^2 \| \y(\x_t) - \y^* \|^2 \\
&\leq \frac{L_f^2}{\mu^2} \left\| \nabla_y f(\z_t) \right\|^2,
\end{align*}
where the last inequality is due to the PL condition (Theorem 2 in ~\cite{karimi2016PL}).

Combining pieces together, we have:
\begin{align*}
    \mb{E} \Phi (\x_{t+1}) - \Phi (\x_t)
    &= \underbrace{\left< \nabla \Phi (\x_t), - \eta_{x} K \mb{E}\left[\u_{x, t} - \e_{x, t} \right] \right>}_{A_1} + \underbrace{\frac{L}{2} \mb{E} \| \eta_{x} K \left(\u_{x, t} - \e_{x, t} \right) \|^2}_{A_2} \\
    &\leq - \frac{1}{2}\eta_{x} K \| \nabla \Phi (\x_t) \|^2 - \frac{1}{4} \eta_{x} K \| \nabla_x f(\z_t) \|^2 + 2L \eta_{x}^2 K^2 \mb{E}\left\| \u_{x, t} \right\|^2 \\
    &\qquad + \eta_{x} K \left(\frac{3}{2} + 2L \eta_{x} K \right) \mb{E}\left\| \bar{\e}_{x, t} \right\|^2 + \eta_{x} K \frac{L_f^2}{\mu^2} \left\| \nabla_y f(\z_t) \right\|^2  + \frac{L \eta_{x}^2 K}{m} \sigma_x^2.
\end{align*}
\end{proof}

\begin{restatable} [One Round Progress for $f$] {lemma} {f}
    \label{lemma:f}
    \begin{align*}
        &f(\z_t) - \mb{E} f(\z_{t+1}) \\
        &\leq \frac{3}{2} \eta_x K \left\| \nabla_x f(\z_t) \right\|^2 + 2L_f \eta_x^2 K^2 \mb{E}\left\| \u_{x, t} \right\|^2 + \eta_x K \left(\frac{1}{2} + 2 L_f \eta_x K\right) \mb{E} \left\| \bar{\e}_{x, t} \right\|^2 + \frac{L_f \eta_x^2 K}{m} \sigma_x^2\\
        &\quad - \frac{1}{2} \eta_y K \| \nabla_y f(\z_t) \|^2 +  2L_f \eta_y^2 K^2 \mb{E} \left\| \u_{y, t} \right\|^2 + \eta_y K \left(\frac{1}{2} + 2 L_f \eta_y K\right) \mb{E}\left\| \bar{\e}_{y, t} \right \|^2 + \frac{L_f \eta_y^2 K}{m} \sigma_y^2.
    \end{align*}
\end{restatable}

\begin{proof}
Similarly, due to $L$-smoothness of $f(\z)$, we have:
\begin{align*}
    &f(\z_t) - \mb{E} f(\z_{t+1}) \leq \eta_x K \mb{E}\left< \nabla_x f(\z_t), \u_{x, t} - \e_{x, t} \right> - \eta_y K \mb{E}\left< \nabla_y f(\z_t), \u_{y, t} - \e_{y, t} \right> \\
    &\quad + \frac{L_f \eta_x^2 K^2}{2} \mb{E} \left\| \u_{x, t} - \e_{x, t} \right\|^2 +  \frac{L_f \eta_y^2 K^2}{2} \mb{E} \left\| \u_{y, t} - \e_{y, t} \right\|^2 \\
    &= \eta_x K \mb{E}\left< \nabla_x f(\z_t), \nabla_x f(\z_t) - \bar{\e}_{x, t} \right> - \eta_y K \mb{E}\left< \nabla_y f(\z_t), \nabla_y f(\z_t) - \bar{\e}_{y, t} \right> \\
    &\quad + \frac{L_f \eta_x^2 K^2}{2} \mb{E} \left\| \u_{x, t} - \e_{x, t} \right\|^2 +  \frac{L_f \eta_y^2 K^2}{2} \mb{E} \left\| \u_{y, t} - \e_{y, t} \right\|^2 \\
    &\leq \frac{3}{2} \eta_x K \left\| \nabla_x f(\z_t) \right\|^2 + \frac{1}{2} \eta_x K \mb{E} \left\| \bar{\e}_{x, t} \right\|^2 - \frac{1}{2} \eta_y K \left\| \nabla_y f(\z_t) \right\|^2 + \frac{1}{2} \eta_y K \mb{E} \left\| \bar{\e}_{y, t} \right \|^2 \\
    &\quad + \frac{L_f \eta_x^2 K^2}{2} \mb{E} \left\| \u_{x, t} - \e_{x, t} \right\|^2 +  \frac{L_f \eta_y^2 K^2}{2} \mb{E} \left\| \u_{y, t} - \e_{y, t} \right\|^2 \\
    &\leq \frac{3}{2} \eta_x K \left\| \nabla_x f(\z_t) \right\|^2 + 2L_f \eta_x^2 K^2 \mb{E}\left\| \u_{x, t} \right\|^2 + \eta_x K \left(\frac{1}{2} + 2 L_f \eta_x K\right) \mb{E} \left\| \bar{\e}_{x, t} \right\|^2 + \frac{L_f \eta_x^2 K}{m} \sigma_x^2\\
    &\quad - \frac{1}{2} \eta_y K \| \nabla_y f(\z_t) \|^2 +  2L_f \eta_y^2 K^2 \mb{E} \left\| \u_{y, t} \right\|^2 + \eta_y K \left(\frac{1}{2} + 2 L_f \eta_y K\right) \mb{E} \left\| \bar{\e}_{y, t} \right\|^2 + \frac{L_f \eta_y^2 K}{m} \sigma_y^2.
\end{align*}
\end{proof}

\begin{restatable} [Bounded Error for FSGDA] {lemma} {error}
    \label{lemma:error_fsgda}
        \begin{align*}
            \mb{E} \| \bar{\e}_{x, t} \|^2 &\leq L_f^2 \bigg[ 40K^2 \eta_{x, l}^2 \left\| \nabla_x f(\z_t) \right\|^2 + 40K^2 \eta_{y, l}^2 \left\| \nabla_y f(\z_t) \right\|^2 + 40K^2 \eta_{x, l}^2 \sigma_{x, G}^2 + 40K^2 \eta_{y, l}^2 \sigma_{y, G}^2 \nonumber \\
            &\quad + 5K \eta_{x, l}^2 \sigma_x^2 + 5K \eta_{y, l}^2 \sigma_y^2 \bigg], \\
            \mb{E} \| \bar{\e}_{y, t} \|^2 &\leq L_f^2 \bigg[ 40K^2 \eta_{x, l}^2 \left\| \nabla_x f(\z_t) \right\|^2 + 40K^2 \eta_{y, l}^2 \left\| \nabla_y f(\z_t) \right\|^2 + 40K^2 \eta_{x, l}^2 \sigma_{x, G}^2 + 40K^2 \eta_{y, l}^2 \sigma_{y, G}^2 \\
            &\quad + 5K \eta_{x, l}^2 \sigma_x^2 + 5K \eta_{y, l}^2 \sigma_y^2 \bigg].
        \end{align*}
\end{restatable}
\begin{proof}
    \begin{align*}
        \mb{E} \| \bar{\e}_{x, t} \|^2 &= \mb{E} \left\| \frac{1}{mK} \sum_{i \in S_t} \sum_{j \in [K]} \left( \nabla_x f_i(\z_t) - \nabla_x f_i(\z_{t, i}^j) \right) \right\|^2 \\
        &\leq \mb{E} \left[ \frac{1}{K} \sum_{i \in S_t} \sum_{j \in [K]} \left\| \left( \nabla_x f_i(\z_t) - \nabla_x f_i(\z_{t, i}^j) \right) \right\|^2 \right] \\
        &\leq \frac{L_f^2}{MK} \sum_{i \in [M]} \sum_{j \in [K]} \mb{E} \left\| \left( \z_t - \z_{t, i}^j \right) \right\|^2 
    \end{align*}

    \begin{align*}
        &\mb{E} \left\| \left( \z_t - \z_{t, i}^{j+1} \right) \right\|^2  = \mb{E} \left[ \left\| \x_{t, i}^j - \x_t - \eta_{x, l} \nabla_x f_i (\z_{t, i}^j, \xi_{t,i}^j) \right\|^2 \right] + \mb{E} \left[ \left\| \y_{t, i}^j - \y_t - \eta_{y, l} \nabla_y f_i (\z_{t, i}^j, \xi_{t,i}^j) \right\|^2 \right]\\
        &\leq \mb{E} \left[ \left\| \x_{t, i}^j - \x_t - \eta_{x, l} \nabla_x f_i (\z_{t, i}^j) \right\|^2 \right] + \mb{E} \left[ \left\| \y_{t, i}^j - \y_t - \eta_{y, l} \nabla_y f_i (\z_{t, i}^j) \right\|^2 \right] + \eta_{x, l}^2 \sigma_x^2 + \eta_{y, l}^2 \sigma_y^2 \\
        &\leq \left(1 + \frac{1}{2K - 1}\right) \left\| \z_{t, i}^j - \z_t \right\|^2 + 2K \eta_{x, l}^2 \left\| \nabla_x f_i \left(\z_{t, i}^j\right) \right\|^2 + 2K \eta_{y, l}^2 \left\| \nabla_y f_i \left(\z_{t, i}^j\right) \right\|^2 + \eta_{x, l}^2 \sigma_x^2 + \eta_{y, l}^2 \sigma_y^2 \\
        &\leq \left(1 + \frac{1}{2K - 1}\right) \left\| \z_{t, i}^j - \z_t \right\|^2 + 4K \eta_{x, l}^2 \left\| \nabla_x f_i \left(\z_{t, i}^j \right) - \nabla_x f_i \left(\z_t \right) \right\|^2 + 4K \left\| \nabla_x f_i \left(\z_t \right) \right\|^2 \\
        &\qquad + 4K \eta_{y, l}^2 \left\| \nabla_y f_i \left(\z_{t, i}^j\right) - \nabla_y f_i \left(\z_t \right) \right\|^2 + 4K \eta_{y, l}^2 \left\| \nabla_y f_i \left(\z_t \right) \right\|^2 + \eta_{x, l}^2 \sigma_x^2 + \eta_{y, l}^2 \sigma_y^2 \\
        &\leq \left(1 + \frac{1}{2K - 1} + 4K \max \{L_f^2 \eta_{x, l}^2 , L_f^2 \eta_{y, l}^2\}\right) \left\| \z_{t, i}^j - \z_t \right\|^2 \\
        &\qquad + 4K \eta_{x, l}^2 \left\| \nabla_x f_i \left(\z_t \right) \right\|^2 + 4K \eta_{y, l}^2 \left\| \nabla_y f_i \left(\z_t \right) \right\|^2 + \eta_{x, l}^2 \sigma_x^2 + \eta_{y, l}^2 \sigma_y^2 \\
        &\leq \left(1 + \frac{1}{K - 1} \right) \left\| \z_{t, i}^j - \z_t \right\|^2 + 4K \left\| \nabla_x f_i \left(\z_t \right) \right\|^2 + 4K \eta_{y, l}^2 \left\| \nabla_y f_i \left(\z_t \right) \right\|^2 + \eta_{x, l}^2 \sigma_x^2 + \eta_{y, l}^2 \sigma_y^2 \\
        &\leq \sum_{\tau = 0}^{j-1} \left(1 + \frac{1}{K - 1} \right)^\tau \left[ 4K \eta_{x, l}^2\left\| \nabla_x f_i \left(\z_t \right) \right\|^2 + 4K \eta_{y, l}^2 \left\| \nabla_y f_i \left(\z_t \right) \right\|^2 + \eta_{x, l}^2 \sigma_x^2 + \eta_{y, l}^2 \sigma_y^2 \right] \\
        &\leq 20K^2 \eta_{x, l}^2 \left\| \nabla_x f_i \left(\z_t \right) \right\|^2 + 20K^2 \eta_{y, l}^2 \left\| \nabla_y f_i \left(\z_t \right) \right\|^2 + 5K \eta_{x, l}^2 \sigma_x^2 + 5K \eta_{y, l}^2 \sigma_y^2 \\
        &\leq 40K^2 \eta_{x, l}^2 \left\| \nabla_x f(\z_t) \right\|^2 + 40K^2 \eta_{y, l}^2 \left\| \nabla_y f(\z_t) \right\|^2 + 40K^2 \eta_{x, l}^2 \sigma_{x, G}^2 + 40K^2 \eta_{y, l}^2 \sigma_{y, G}^2 \\
        &\quad + 5K \eta_{x, l}^2 \sigma_x^2 + 5K \eta_{y, l}^2 \sigma_y^2,
    \end{align*}
where the first inequality is due to bounded variance of stochastic gradient, the second and third inequalities follow from the fact $\| \a + \b \|^2 \leq \left(1 + \frac{1}{\epsilon}\right) \| \a \|^2 + \left(1 + \epsilon \right) \| \b \|^2$, the forth inequality is due to smoothness of $f$ in $x$ and $y$, fifth inequality holds if
\begin{align}
    4K \max \{L_f^2 \eta_{x, l}^2 , L_f^2 \eta_{y, l}^2\} \leq \frac{1}{2(K-1)(2K-1)}, \label{fsgda_lr1}
\end{align}
the second last inequality follows from the $\sum_{\tau = 0}^{j-1} \left(1 + \frac{1}{K - 1} \right)^\tau \leq (K - 1) \left[ \left(1 + \frac{1}{K - 1} \right)^K - 1 \right] \leq 5K$, and the last inequality is due to the Assumption~\ref{a_dissimilarity}.

Plugging into the bound of $\| \bar{\e}_{x, t} \|^2 $, we have:
\begin{align*}
    \| \bar{\e}_{x, t} \|^2 &\leq L_f^2 \bigg[ 40K^2 \eta_{x, l}^2 \left\| \nabla_x f(\z_t) \right\|^2 + 40K^2 \eta_{y, l}^2 \left\| \nabla_y f(\z_t) \right\|^2 + 40K^2 \eta_{x, l}^2 \sigma_{x, G}^2 + 40K^2 \eta_{y, l}^2 \sigma_{y, G}^2 \nonumber\\
    &\quad + 5K \eta_{x, l}^2 \sigma_x^2 + 5K \eta_{y, l}^2 \sigma_y^2  \bigg].
\end{align*}

The bound of $\| \bar{\e}_{y, t} \|^2$ follows from the similar proof.

\end{proof}

\convergence*

\begin{proof}
    Define potential function $\mathcal{L}_{t} = \Phi (\x_t) - \frac{1}{10} f(\z_t)$, 
    \begin{align*}
        &\mb{E} L_{t+1} - \mathcal{L}_{t} 
        = \mb{E} \Phi (\x_{t+1}) - \Phi (\x_t) + \frac{1}{10} \left(f(\z_t) - \mb{E} f(\z_{t+1})\right) \\
        &\leq - \frac{1}{2}\eta_{x} K \| \nabla \Phi (\x_t) \|^2 
        - \frac{1}{10} \eta_{x} K \| \nabla_x f(\z_t) \|^2 + \left( (2L +\frac{1}{5} L_f) \eta_{x}^2 K^2 \right) \mb{E}\left\| \u_{x, t} \right\|^2 \\
        &\quad - \eta_y K \left(\frac{1}{20} - \frac{\eta_{x}}{\eta_y} \frac{L_f^2}{\mu^2}\right) \| \nabla_y f(\z_t) \|^2 + \frac{1}{5} L_f \eta_y^2 K^2 \mb{E}\left\| \u_{y, t} \right\|^2 \\
        &\quad + \eta_{x} K \left(\frac{31}{20} + (2L + \frac{1}{5}L_f) \eta_{x} K \right) \mb{E}\left\| \bar{\e}_{x, t} \right\|^2 + \eta_y K \left(\frac{1}{20} + \frac{1}{5} L_f \eta_y K\right) \mb{E}\left\| \bar{\e}_{y, t} \right \|^2 \\
        &\quad + \frac{\eta_{x}^2 K}{m} \left(L + \frac{L_f}{10}\right) \sigma_x^2 + \frac{L_f \eta_y^2 K}{10m} \sigma_y^2 \\
        &\leq - \frac{1}{2}\eta_{x} K \| \nabla \Phi (\x_t) \|^2 
        - \eta_{x} K \underbrace{\left(\frac{1}{10} - 2(2L +\frac{1}{5} L_f) \eta_{x} K \right)}_{a_1} \left\| \nabla_x f(\z_t) \right\|^2 \\
        &\quad - \eta_y K \underbrace{\left(\frac{1}{20} - \frac{2}{5} L_f \eta_y K - \frac{\eta_{x}}{\eta_y} \frac{L_f^2}{\mu^2}\right)}_{a_2} \left\| \nabla_y f(\z_t) \right\|^2 \\
        &\quad + \left( (2L +\frac{1}{5} L_f) \eta_{x}^2 K^2 \right) \left(1 - \frac{m}{M}\right) \frac{2}{m} \sigma_{x, G}^2 + \frac{1}{5} L_f \eta_y^2 K^2 \left(1 - \frac{m}{M}\right) \frac{2}{m} \sigma_{y, G}^2 \\
        &\quad + \eta_{x} K \underbrace{\left(\frac{31}{20} + (2L + \frac{1}{5}L_f) \eta_{x} K \right)}_{a_3} \mb{E}\left\| \bar{\e}_{x, t} \right\|^2 + \eta_y K \underbrace{\left(\frac{1}{20} + \frac{1}{5} L_f \eta_y K\right)}_{a_4} \mb{E}\left\| \bar{\e}_{y, t} \right\|^2 \\
        &\quad + \frac{\eta_{x}^2 K}{m} \left(L + \frac{L_f}{10}\right) \sigma_x^2 + \frac{L_f \eta_y^2 K}{10m} \sigma_y^2 \\
        &\leq - \frac{1}{2}\eta_{x} K \| \nabla \Phi (\x_t) \|^2 + \frac{\eta_{x}^2 K}{m} \left(L + \frac{L_f}{10}\right) \sigma_x^2 + \frac{L_f \eta_y^2 K}{10m} \sigma_y^2 \\
        &\quad + \left( (2L +\frac{1}{5} L_f) \eta_{x}^2 K^2 \right) \left(1 - \frac{m}{M}\right) \frac{2}{m} \sigma_{x, G}^2 + \frac{1}{5} L_f \eta_y^2 K^2 \left(1 - \frac{m}{M}\right) \frac{2}{m} \sigma_{y, G}^2 \\
        &\quad + K \left( a_3 L_f^2 \eta_x + a_4 \eta_y L_f^2 \right) \left[ 40K^2 \eta_{x, l}^2 \sigma_{x, G}^2 + 40K^2 \eta_{y, l}^2 \sigma_{y, G}^2 + 5K \eta_{x, l}^2 \sigma_x^2 + 5K \eta_{y, l}^2 \sigma_y^2  \right],
    \end{align*}
    where the second inequality is due to $\mb{E} \| \u_{x,t} \|^2 \leq 2\| \nabla_x f(\z_t) \|^2 +2 \left(1 - \frac{m}{M}\right) \frac{\sigma_{x, G}^2}{m}$ and $\mb{E} \| \u_{y,t} \|^2 \leq 2\| \nabla_y f(\z_t) \|^2 + 2\left(1 - \frac{m}{M}\right) \frac{\sigma_{y, G}^2}{m}$,
    the last inequality follows from the conditions:
    \begin{align}
        &a_1 - a_3 40 L_f^2 K^2 \eta_{x, l}^2 - \frac{\eta_y}{\eta_x} a_4 40 L_f^2 K^2 \eta_{x, l}^2 \geq 0, \label{fsgda_lr2}\\ 
        &a_2 - a_3 \frac{\eta_x}{\eta_y} 40 L_f^2 K^2 \eta_{y, l}^2 - a_4 40 L_f^2 K^2 \eta_{y, l}^2 \geq 0. \label{fsgda_lr3}
    \end{align}
    
    Telescoping and rearranging, we have:
    \begin{align*}
        &\frac{1}{T} \sum_{t=0}^{T-1} \| \nabla \Phi (\x_t) \|^2 \leq \frac{2 \left(\mathcal{L}_0 - \mathcal{L}_{*} \right)}{\eta_x K T} + \frac{2 \eta_{x}}{m} \left(L + \frac{L_f}{100}\right) \sigma_x^2 + \frac{ L_f \eta_y^2 }{5 m \eta_x} \sigma_y^2 \\
        &\quad + \left( (2L +\frac{1}{5} L_f) \eta_{x} K \right) \left(1 - \frac{m}{M}\right) \frac{2}{m} \sigma_{x, G}^2 + \frac{1}{5} L_f \eta_y K \frac{\eta_y}{\eta_x} \left(1 - \frac{m}{M}\right) \frac{2}{m} \sigma_{y, G}^2 \\
        &\quad + 2 \left( a_3 L_f^2 + a_4 \frac{\eta_y}{\eta_x} L_f^2 \right) \left[ 40K^2 \eta_{x, l}^2 \sigma_{x, G}^2 + 40K^2 \eta_{y, l}^2 \sigma_{y, G}^2 + 5K \eta_{x, l}^2 \sigma_x^2 + 5K \eta_{y, l}^2 \sigma_y^2  \right].
    \end{align*}
    
\end{proof}

\subsection{Proof for \algn Option I} \label{subsec: fedvr2}
For \algn Option I, the update rule is:
\begin{align*}
    \x_{t+1} &= \x_t - \eta_{x, g} \eta_{x, l} \left[\frac{1}{m} \sum_{i \in S_t} \sum_{j \in [K]} \left( \nabla_x f_i(\z_{t, i}^j, \xi_{t,i}^j) - \v_{x}^i + \bar{\v}_{x, t} \right) \right],  \\
    \y_{t+1} &= \y_t + \eta_{y, g} \eta_{y, l} \left[\frac{1}{m} \sum_{i \in S_t} \sum_{j \in [K]} \left( \nabla_y f_i(\z_{t, i}^j, \xi_{t,i}^j) - \v_{y}^i + \bar{\v}_{y, t} \right) \right], \\
    \e_{x, t} &= \frac{1}{mK} \sum_{i \in S_t} \sum_{j \in [K]} \left[ \nabla_x f_i(\z_t) - \left( \nabla_x f_i(\z_{t, i}^j, \xi_{t,i}^j) - \v_{x}^i + \bar{\v}_{x, t} \right) \right] \\
    \bar{\e}_{x, t} &= \mb{E} [\e_{x, t}] = \frac{1}{mK} \sum_{i \in S_t} \sum_{j \in [K]} \left( \nabla_x f_i(\z_t) - \nabla_x f_i(\z_{t, i}^j) \right), \\
    \e_{y, t} &= \frac{1}{mK} \sum_{i \in S_t} \sum_{j \in [K]} \left[ \nabla_y f_i(\z_t) - \left( \nabla_y f_i(\z_{t, i}^j, \xi_{t,i}^j) - \v_{y}^i + \bar{\v}_{y, t} \right) \right] \\
    \bar{\e}_{y, t} &= \mb{E} [\e_{y, t}] = \frac{1}{mK} \sum_{i \in S_t} \sum_{j \in [K]} \left( \nabla_y f_i(\z_t) - \nabla_y f_i(\z_{t, i}^j) \right),
\end{align*}
where we define $\v_{x}^i = \nabla_x f_i(\w_{t, i}, \xi)$ and $\bar{\v}_{x, t} = \frac{1}{M} \sum_{i \in [M]} \v_{x}^i$ with a sequence of parameters $\w_{t, i}$ such that 
\begin{align*}
    \w_{t, i} := 
    \begin{cases}
    \z_{t-1}, \textit{ if }i \in S_{t-1}, \\
    \w_{t-1, i}, \textit{ otherwise}.
    \end{cases}
\end{align*}

We further have the following definition for notational clarity:
\begin{align*}
    \Delta \x_t &= \frac{1}{mK} \sum_{i \in S_t} \sum_{j \in [K]} \left[\nabla_x f_i(\z_{t, i}^j, \xi_{t, i}^j) - \v_{x}^i + \bar{\v}_{x, t} \right], \\
    \Delta \y_t &= \frac{1}{mK} \sum_{i \in S_t} \sum_{j \in [K]} \left[\nabla_y f_i(\z_{t, i}^j, \xi_{t, i}^j) - \v_{y}^i + \bar{\v}_{y, t} \right], \\
    \Psi_t &= \frac{1}{MK} \sum_{i \in [M]} \sum_{j \in [K]} \mb{E} \left\| \z_{t, i}^j - \z_t \right\|^2, \\
    \Gamma_t &= \frac{1}{M} \sum_{i \in [M]} \mb{E} \left\| \w_{t, i} - \z_{t} \right\|^2.
\end{align*}

\begin{restatable} [Iterative Control Variate] {lemma} {controlVariant}
    \label{lemma:controlVariant}
    \begin{align*}
        \Gamma_t
        &= \left(1 - \frac{m}{2M}\right) \Gamma_{t-1} + \left( \frac{m}{M} + \frac{M}{m} - 1 \right) \mb{E} \left\| \z_t - \z_{t-1} \right\|^2.
    \end{align*}
\end{restatable}

\begin{proof}
\begin{align*}
    \Gamma_t &= \frac{1}{M} \sum_{i \in [M]} \mb{E} \left\| \w_{t, i} - \z_{t} \right\|^2 \\
    &= \left(1 - \frac{m}{M}\right) \frac{1}{M} \sum_{i \in [M]} \mb{E} \left\| \w_{t-1, i} - \z_{t} \right\|^2 + \frac{m}{M} \mb{E} \left\| \z_{t-1} - \z_{t} \right\|^2 \\
    &\leq \left(1 - \frac{m}{M}\right) \left(1 + \frac{1}{b} \right) \Gamma_{t-1} + \left[ \left(1 - \frac{m}{M}\right) \left(1 + b \right) + \frac{m}{M} \right] \mb{E} \left\| \z_t - \z_{t-1} \right\|^2 \\
    &= \left(1 - \frac{m}{2M}\right) \Gamma_{t-1} + \left( \frac{m}{M} + \frac{M}{m} - 1 \right) \mb{E} \left\| \z_t - \z_{t-1} \right\|^2,
\end{align*}
where we set $b = \frac{2M}{m} - 1$.
\end{proof}

\begin{restatable} [Local Step Distance for \algn Option I] {lemma} {distance}
    \label{lemma:distance_vr}
    $\forall i \in [M], j \in [K]$, we can bound the local step distance as follows:
    \begin{align*}
        \frac{1}{M} \sum_{i \in [M]}  \mb{E} \left\| \left( \z_{t, i}^{j} - \z_t \right) \right\|^2
        &\leq 160 K^2 \left(\eta_{x, l}^2 + \eta_{y, l}^2 \right) L_f^2 \Gamma_t + 10 K^2 \left(\eta_{x, l}^2 \sigma_x^2 + \eta_{y, l}^2 \sigma_y^2\right) \\
        &\quad + 40 K^2 \left(\eta_{x, l}^2 \mb{E} \left\| \nabla_x f(\z_t) \right\|^2 + \eta_{y, l}^2 \mb{E} \left\| \nabla_y f(\z_t) \right\|^2 \right).
    \end{align*}
\end{restatable}

\begin{proof}
First, we bound the local update as follows:
    \begin{align*}
        &\mb{E} \left\| \left( \nabla_x f_i (\z_{t, i}^j, \xi_{t, i}^j) - \v_{x}^i + \bar{\v}_{x, t} \right) \right\|^2 \\
        &\leq 4 \bigg[ \mb{E} \left\| \nabla_x f_i(\z_{t, i}^j) - \nabla_x f_i(\z_t) \right\|^2 + \mb{E} \left\| \mb{E}[\v_{x}^i] - \nabla_x f_i(\z_{t}) \right\|^2 + \mb{E} \left\| \mb{E}[\bar{\v}_{x, t}] - \nabla_x f(\z_{t}) \right\|^2 \\
        &\quad + \left\| \nabla_x f(\z_{t}) \right\|^2 \bigg] + \sigma_x^2 \\
        &\leq 4 L_f^2 \mb{E} \left\| \z_{t, i}^j - \z_t \right\|^2 + 4 L_f^2 \mb{E} \left\| \w_{t, i} - \z_t \right\|^2 + 4 L_f^2 \mb{E} \left\| \mb{E}[\bar{\v}_{x, t}] - \nabla_x f(\z_{t}) \right\|^2 \\
        &\quad + 4 \mb{E} \left\| \nabla_x f(\z_t) \right\|^2 + \sigma_x^2.
    \end{align*}
That is, 
    \begin{align*}
        &\frac{1}{M} \sum_{i \in [M]} \mb{E} \left\| \left( \nabla_x f_i (\z_{t, i}^j, \xi_{t, i}^j) - \v_{x}^i + \bar{\v}_{x, t} \right) \right\|^2 \\
        &\leq 4 L_f^2 \frac{1}{M} \sum_{i \in [M]} \mb{E} \left\| \z_{t, i}^j - \z_t \right\|^2 + 8 L_f^2 \Gamma_t + \sigma_x^2 + 4 \mb{E} \left\| \nabla_x f(\z_t) \right\|^2.
    \end{align*}

    \begin{align*}
        &\frac{1}{M} \sum_{i \in [M]} \mb{E} \left[ \left\| \x_{t, i}^{j+1} - \x_t \right\|^2 \right] = \frac{1}{M} \sum_{i \in [M]} \mb{E} \left[ \left\| \x_{t, i}^j - \x_t - \eta_{x, l} \left( \nabla_x f_i (\z_{t, i}^j, \xi_{t,i}^j) - \v_{x, t}^i + \bar{\v}_{x, t} \right) \right\|^2 \right] \\
        &\leq \left(1 + \frac{1}{2K - 1}\right) \frac{1}{M} \sum_{i \in [M]} \mb{E} \left\| \x_{t, i}^j - \x_t \right\|^2 + 2K \eta_{x, l}^2 \frac{1}{M} \sum_{i \in [M]} \mb{E} \left\| \nabla_x f_i (\z_{t, i}^j, \xi_{t, i}^j) - \v_{x, t}^i + \bar{\v}_{x, t} \right\|^2 \\
        &\leq \left(1 + \frac{1}{2K - 1} + 8K L_f^2 \eta_{x, l}^2 \right) \frac{1}{M} \sum_{i \in [M]} \mb{E} \left\| \x_{t, i}^j - \x_t \right\|^2 + 32 K \eta_{x, l}^2 L_f^2 \Gamma_t \\
        &\quad + 2 K \eta_{x, l}^2 \sigma_x^2 + 8 K \eta_{x, l}^2 \mb{E} \left\| \nabla_x f(\z_t) \right\|^2.
    \end{align*}

We can bound $\left\| \y_{t, i}^{j+1} - \y_t \right\|^2$ in the same way, and then we have
    \begin{align*}
        &\frac{1}{M} \sum_{i \in [M]}  \mb{E} \left\| \left( \z_{t, i}^{j+1} - \z_t \right) \right\|^2 \\
        &\leq \left(1 + \frac{1}{2K - 1} + 8K L_f^2 \max \{\eta_{x, l}^2, \eta_{y, l}^2 \} \right) \frac{1}{M} \sum_{i \in [M]} \mb{E} \left\| \z_{t, i}^j - \z_t \right\|^2 + \bigg[32 K \left(\eta_{x, l}^2 + \eta_{y, l}^2 \right) L_f^2 \Gamma_t \\
        &\quad + 2 K \left(\eta_{x, l}^2 \sigma_x^2 + \eta_{y, l}^2 \sigma_y^2\right) + 8 K \left(\eta_{x, l}^2 \mb{E} \left\| \nabla_x f(\z_t) \right\|^2 + \eta_{y, l}^2 \mb{E} \left\| \nabla_y f(\z_t) \right\|^2 \right)\bigg] \\
        &\leq \left(1 + \frac{1}{K - 1}\right)\frac{1}{M} \sum_{i \in [M]} \mb{E} \left\| \z_{t, i}^j - \z_t \right\|^2 \bigg[32 K \left(\eta_{x, l}^2 + \eta_{y, l}^2 \right) L_f^2 \Gamma_t \\
        &\quad + 2 K \left(\eta_{x, l}^2 \sigma_x^2 + \eta_{y, l}^2 \sigma_y^2\right) + 8 K \left(\eta_{x, l}^2 \mb{E} \left\| \nabla_x f(\z_t) \right\|^2 + \eta_{y, l}^2 \mb{E} \left\| \nabla_y f(\z_t) \right\|^2 \right)\bigg] \\
        &\leq \sum_{\tau = 0}^{j-1} \left(1 + \frac{1}{K - 1} \right)^\tau \bigg[32 K \left(\eta_{x, l}^2 + \eta_{y, l}^2 \right) L_f^2 \Gamma_t \\
        &\quad + 2 K \left(\eta_{x, l}^2 \sigma_x^2 + \eta_{y, l}^2 \sigma_y^2\right) + 8 K \left(\eta_{x, l}^2 \mb{E} \left\| \nabla_x f(\z_t) \right\|^2 + \eta_{y, l}^2 \mb{E} \left\| \nabla_y f(\z_t) \right\|^2 \right)\bigg] \\
        &\leq 160 K^2 \left(\eta_{x, l}^2 + \eta_{y, l}^2 \right) L_f^2 \Gamma_t + 10 K^2 \left(\eta_{x, l}^2 \sigma_x^2 + \eta_{y, l}^2 \sigma_y^2\right) \\
        &\quad + 40 K^2 \left(\eta_{x, l}^2 \mb{E} \left\| \nabla_x f(\z_t) \right\|^2 + \eta_{y, l}^2 \mb{E} \left\| \nabla_y f(\z_t) \right\|^2 \right).
    \end{align*}
    The learning rates should satisfy
    \begin{align}
        4K \max \{L_f^2 \eta_{x, l}^2 , L_f^2 \eta_{y, l}^2\} \leq \frac{1}{2(K-1)(2K-1)}, \label{sagda1_lr1}
    \end{align}

% $\bar{\v}_{x, t} = \frac{1}{m} \sum_{i \in S_t} \nabla_x f_i(\z_t, \xi_{t, i}) $ and $\v_{x, t}^i = \nabla_x f_i (\z_t, \xi_{t,i})$; 
% $\bar{\v}_{y, t} = \frac{1}{m} \sum_{i \in S_t} \nabla_y f_i(\z_t, \xi_{t, i})$ and $\bar{\v}_{y, t}^i = \nabla_y f_i (\z_t, \xi_{t,i})$; 
% where the first inequality is due to bounded variance of stochastic gradient, the second and third inequalities follow from the fact $\| \a + \b \|^2 \leq \left(1 + \frac{1}{\epsilon}\right) \| \a \|^2 + \left(1 + \epsilon \right) \| \b \|^2$, the forth inequality is due to smoothness of $f$ in $x$ and $y$, fifth inequality holds if $4K \max \{L_f^2 \eta_{x, l}^2 , L_f^2 \eta_{y, l}^2\} \leq \frac{1}{2(K-1)(2K-1)}$, and the last inequality follows from the $\sum_{\tau = 0}^{j-1} \left(1 + \frac{1}{K - 1} \right)^\tau \leq (K - 1) \left[ \left(1 + \frac{1}{K - 1} \right)^K - 1 \right] \leq 5K$.

\end{proof}

\begin{restatable} {lemma} {deltaXY_vr}
    \label{lemma:deltaXY_vr}
        \begin{align*}
            \mb{E}\| \Delta \x_t \|^2 &\leq 4 L_f^2 \Psi_t + 4 L_f^2 \Gamma_t + 4 \left\| \nabla_x f(\z_t) \right\|^2 + \frac{9}{mK} \sigma_x^2, \\
            \mb{E}\| \Delta \y_t \|^2 &\leq 4 L_f^2 \Psi_t + 4 L_f^2 \Gamma_t + 4 \left\| \nabla_y f(\z_t) \right\|^2 + \frac{9}{mK} \sigma_y^2,\\
            \mb{E}\| \z_{t+1} - \z_t \|^2 &\leq 4 L_f^2 K^2 \left(\eta_x^2 + \eta_y^2\right) \Psi_t + 4 L_f^2 K^2 \left(\eta_x^2 + \eta_y^2\right) \Gamma_t \\
            &\quad + 4 K^2 \left( \eta_x^2\left\| \nabla_x f(\z_t) \right\|^2 + \eta_y^2 \left\| \nabla_y f(\z_t) \right\|^2 \right) + \frac{9K}{m} \left(\eta_x^2 \sigma_x^2 + \eta_y^2 \sigma_y^2 \right).
        \end{align*}
\end{restatable}
\begin{proof}
    \begin{align*}
        &\mb{E} \|\Delta \x_t \|^2 
        \leq \mb{E} \left\| \frac{1}{mK} \sum_{i \in S_t} \sum_{j \in [K]} \left[\nabla_x f_i(\z_{t, i}^j) - \mb{E}[\v_{x}^i] + \mb{E}[\bar{\v}_{x, t}] \right] \right\|^2 + \frac{9}{mK} \sigma_x^2 \\
        &\leq \frac{4}{MK} \sum_{i \in [M]} \sum_{j \in [K]} \bigg[ \mb{E} \left\| \nabla_x f_i(\z_{t, i}^j) - \nabla_x f_i(\z_t) \right\|^2 + \mb{E} \left\| \mb{E}[\v_{x}^i] - \nabla_x f_i(\z_{t}) \right\|^2 \\
        &\quad + \mb{E} \left\| \mb{E}[\bar{\v}_{x, t}] - \nabla_x f(\z_{t}) \right\|^2 + \left\| \nabla_x f(\z_{t}) \right\|^2 \bigg] + \frac{9}{mK} \sigma_x^2 \\
        &\leq \frac{4}{MK} \sum_{i \in [M]} \sum_{j \in [K]} \bigg[ L_f^2 \mb{E} \left\| \z_{t, i}^j - \z_t \right\|^2 + L_f^2 \mb{E} \left\|\w_{t, i} - \z_{t} \right\|^2 + \left\| \nabla_x f(\z_t) \right\|^2 \bigg] + \frac{9}{mK} \sigma_x^2 \\
        &= 4 L_f^2 \Psi_t + 4 L_f^2 \Gamma_t + 4 \left\| \nabla_x f(\z_t) \right\|^2 + \frac{9}{mK} \sigma_x^2,
    \end{align*}
    $\mb{E}[\v_{x, t}^i] = \nabla_x f_i(\z_t)$ and $\mb{E}[\bar{\v}_{x, t}] = \nabla_x f(\z_t)$
    where the second inequality is due to Lemma 4 in \cite{Karimireddy2020SCAFFOLD}).

    The bound of $\| \left(\u_{y, t} - \e_{y, t} \right) \|^2$ follows from the similar proof.
\end{proof}

\begin{restatable} [Bounded Error for \algn Option I] {lemma} {error}
    \label{lemma:error_vr}
        \begin{eqnarray*}
            \mb{E} \| \bar{\e}_{x, t} \|^2 &\leq L_f^2 \Psi_t, \\
            \mb{E} \| \bar{\e}_{y, t} \|^2 &\leq L_f^2 \Psi_t.
        \end{eqnarray*}
\end{restatable}
\begin{proof}
    \begin{align*}
        \mathbb{E} \| \bar{\e}_{x, t} \|^2 &= \mathbb{E} \left\| \frac{1}{mK} \sum_{i \in S_t} \sum_{j \in [K]} \left( \nabla_x f_i(\z_t) - \nabla_x f_i(\z_{t, i}^j) \right) \right\|^2 \\
        &\leq \frac{1}{mK}  \mathbb{E} \left[\sum_{i \in S_t} \sum_{j \in [K]} \left\| \left( \nabla_x f_i(\z_t) - \nabla_x f_i(\z_{t, i}^j) \right) \right\|^2 \right] \\
        &\leq \frac{L_f^2}{MK} \sum_{i \in [M], j \in [K]} \mathbb{E} \left\| \left( \z_t - \z_{t, i}^j \right) \right\|^2 \\
        &= L_f^2 \Psi_t.
    \end{align*}
    $\mb{E} \| \bar{\e}_{y, t} \|^2$ has the same bounds.
\end{proof}

\convergencesagda*
\begin{proof}
    Similar to the bound of $\Phi$ and $f$ in \eqref{lemma:phi} and \eqref{lemma:f}, we have the following results:
    \begin{align*}
        \mb{E} \Phi (\x_{t+1}) - \Phi (\x_t)
        &\leq - \frac{1}{2}\eta_{x} K \| \nabla \Phi (\x_t) \|^2 - \frac{1}{4} \eta_{x} K \| \nabla_x f(\z_t) \|^2 + \frac{3}{2} \eta_{x} K \mb{E}\left\| \bar{\e}_{x, t} \right\|^2 \\
        &\quad + \eta_{x} K \frac{L_f^2}{\mu^2} \left\| \nabla_y f(\z_t) \right\|^2 + \frac{1}{2} L \eta_x^2 K^2 \mb{E}\left\| \u_{x,t} - \e_{x, t} \right\|^2.
    \end{align*}

    \begin{align*}
        f(\z_t) - \mb{E} f(\z_{t+1})
        &\leq \frac{3}{2} \eta_x K \left\| \nabla_x f(\z_t) \right\|^2 + \frac{1}{2} \eta_x K \mb{E} \left\| \bar{\e}_{x, t} \right\|^2 + \frac{1}{2} \eta_y K \mb{E}\left\| \bar{\e}_{y, t} \right \|^2 - \frac{1}{2} \eta_y K \| \nabla_y f(\z_t) \|^2 \\
        &\quad + \frac{1}{2} L_f \eta_x^2 K^2 \left\| \u_{x, t} - \e_{x, t} \right\|^2 + \frac{1}{2} L_f \eta_y^2 K^2 \left\| \u_{y, t} - \e_{y, t} \right\|^2.
    \end{align*}

    Define potential function $\mathcal{L}_t = \Phi (\x_t) - \frac{1}{10} f(\z_t)$, 
    \begin{align*}
        &\mb{E} \mathcal{L}_{t+1} - \mathcal{L}_{t} 
        = \mb{E} \Phi (\x_{t+1}) - \Phi (\x_t) + \frac{1}{10} \left(f(\z_t) - \mb{E} f(\z_{t+1})\right) \\
        &\leq - \frac{1}{2}\eta_{x} K \| \nabla \Phi (\x_t) \|^2 
        - \frac{1}{10} \eta_{x} K \left\| \nabla_x f(\z_t) \right\|^2 - \eta_y K \left(\frac{1}{20} - \frac{\eta_{x}}{\eta_y} \frac{L_f^2}{\mu^2}\right) \left\| \nabla_y f(\z_t) \right\|^2 \\
        &\quad + \frac{31}{20} \eta_{x} K \left\| \bar{\e}_{x, t} \right\|^2 + \frac{1}{20} \eta_y K \left\| \bar{\e}_{y, t} \right \|^2 + \frac{1}{2}\left(L + \frac{L_f}{10}\right) \eta_x^2 K^2 \mb{E} \left\| \Delta \x_{t} \right\|^2 + \frac{1}{20} L_f \eta_y^2 K^2 \mb{E}\left\| \Delta \y_t \right\|^2 \\
        &\leq - \frac{1}{2}\eta_{x} K \| \nabla \Phi (\x_t) \|^2 
        - \frac{1}{10} \eta_{x} K \left\| \nabla_x f(\z_t) \right\|^2 - \eta_y K \left(\frac{1}{20} - \frac{\eta_{x}}{\eta_y} \frac{L_f^2}{\mu^2}\right) \left\| \nabla_y f(\z_t) \right\|^2 \\
        &\quad + K L_f^2 \left( \frac{31}{20} \eta_{x} + \frac{1}{20} \eta_y \right) \Psi_t + \frac{1}{2}\left(L + \frac{L_f}{10}\right) \mb{E} \left\| \z_{t+1} - \z_t \right\|^2 
    \end{align*}

    \begin{align*}
        &\left(\mb{E} \mathcal{L}_{t+1} + \alpha \Gamma_{t+1} \right) - \left(\mathcal{L}_{t} + \alpha \Gamma_t \right) \\
        &\leq - \frac{1}{2}\eta_{x} K \| \nabla \Phi (\x_t) \|^2 - \frac{1}{10} \eta_{x} K \left\| \nabla_x f(\z_t) \right\|^2 - \eta_y K \left(\frac{1}{20} - \frac{\eta_{x}}{\eta_y} \frac{L_f^2}{\mu^2}\right) \left\| \nabla_y f(\z_t) \right\|^2 \\
        &\quad + K L_f^2 \left( \frac{31}{20} \eta_{x} + \frac{1}{20} \eta_y \right) \Psi_t + \frac{1}{2}\left(L + \frac{L_f}{10}\right) \mb{E} \left\| \z_{t+1} - \z_t \right\|^2 + \alpha \Gamma_{t+1} - \alpha \Gamma_t \\
        &\leq - \frac{1}{2}\eta_{x} K \| \nabla \Phi (\x_t) \|^2 - \frac{1}{10} \eta_{x} K \left\| \nabla_x f(\z_t) \right\|^2 - \eta_y K \left(\frac{1}{20} - \frac{\eta_{x}}{\eta_y} \frac{L_f^2}{\mu^2}\right) \left\| \nabla_y f(\z_t) \right\|^2 \\
        &\quad + \underbrace{K L_f^2 \left( \frac{31}{20} \eta_{x} + \frac{1}{20} \eta_y \right)}_{a_1} \Psi_t + \underbrace{\left[ \frac{1}{2}\left(L + \frac{L_f}{10}\right) + \alpha \left( \frac{m}{M} + \frac{M}{m} - 1 \right)\right]}_{a_2} \mb{E} \left\| \z_{t+1} - \z_t \right\|^2 - \alpha \frac{m}{2M} \Gamma_{t} \\
        &\leq - \frac{1}{2}\eta_{x} K \| \nabla \Phi (\x_t) \|^2 - \frac{1}{10} \eta_{x} K \left\| \nabla_x f(\z_t) \right\|^2 - \eta_y K \left(\frac{1}{20} - \frac{\eta_{x}}{\eta_y} \frac{L_f^2}{\mu^2}\right) \left\| \nabla_y f(\z_t) \right\|^2 \\
        &\quad + \left[a_1 + a_2 4 L_f^2 K^2 \left(\eta_x^2 + \eta_y^2\right) \right] \Psi_t + \left[4 a_2 L_f^2 K^2 \left(\eta_x^2 + \eta_y^2\right) - \alpha \frac{m}{2M}\right] \Gamma_t \\
        &\quad + a_2 \left[ 4 K^2 \left( \eta_x^2\left\| \nabla_x f(\z_t) \right\|^2 + \eta_y^2 \left\| \nabla_y f(\z_t) \right\|^2 \right) + \frac{9K}{m} \left(\eta_x^2 \sigma_x^2 + \eta_y^2  \sigma_y^2 \right)\right] \\
        &\leq - \frac{1}{2}\eta_{x} K \| \nabla \Phi (\x_t) \|^2 - \left[ \frac{1}{10} \eta_{x} K - 4 a_2 K^2 \eta_x^2 \right] \left\| \nabla_x f(\z_t) \right\|^2 \\
        &\quad - \left[ \eta_y K \left(\frac{1}{20} - \frac{\eta_{x}}{\eta_y} \frac{L_f^2}{\mu^2}\right) - 4 a_2 K^2 \eta_y^2 \right] \left\| \nabla_y f(\z_t) \right\|^2 + \left[a_1 + a_2 4 L_f^2 K^2 \left(\eta_x^2 + \eta_y^2\right) \right] \times \\
        &\quad \left[ 10 K^2 \left(\eta_{x, l}^2 \sigma_x^2 + \eta_{y, l}^2 \sigma_y^2\right) + 40 K^2 \left(\eta_{x, l}^2 \mb{E} \left\| \nabla_x f(\z_t) \right\|^2 + \eta_{y, l}^2 \mb{E} \left\| \nabla_y f(\z_t) \right\|^2 \right)\right] \\
        &\quad - \left[\alpha \frac{m}{2M} - 4 a_2 L_f^2 K^2 \left(\eta_x^2 + \eta_y^2\right) - \left(a_1 + a_2 4 L_f^2 K^2 \left(\eta_x^2 + \eta_y^2\right) \right) 160 K^2 \left(\eta_{x, l}^2 + \eta_{y, l}^2 \right) L_f^2\right] \Gamma_t \\
        &\quad + a_2 \frac{9K}{m} \left(\eta_x^2 \sigma_x^2 + \eta_y^2  \sigma_y^2 \right),
    \end{align*}
    where we can set $\alpha = \frac{M}{m}$ and requires the learning rates $\eta_x, \eta_y$ and $\eta_{x, l}, \eta_{y, l}$ satisfy
    \begin{align}
        &\left[\alpha \frac{m}{2M} - 4 a_2 L_f^2 K^2 \left(\eta_x^2 + \eta_y^2\right) - \left(a_1 + a_2 4 L_f^2 K^2 \left(\eta_x^2 + \eta_y^2\right) \right) 160 K^2 \left(\eta_{x, l}^2 + \eta_{y, l}^2 \right) L_f^2\right] \geq 0, \label{sagda1_lr2} \\
        &\left[ \frac{1}{10} \eta_{x} K - 4 a_2 K^2 \eta_x^2 \right] - \left[a_1 + a_2 4 L_f^2 K^2 \left(\eta_x^2 + \eta_y^2\right) \right] 40 K^2 \eta_{x, l}^2 \geq 0, \label{sagda1_lr3}\\
        &\left[ \eta_y K \left(\frac{1}{20} - \frac{\eta_{x}}{\eta_y} \frac{L_f^2}{\mu^2}\right) - 4 a_2 K^2 \eta_y^2 \right] - \left[a_1 + a_2 4 L_f^2 K^2 \left(\eta_x^2 + \eta_y^2\right) \right] 40 K^2 \eta_{y, l}^2 \geq 0. \label{sagda1_lr4}
    \end{align}

    \begin{align*}
        &\left(\mb{E} \mathcal{L}_{t+1} + \alpha \Gamma_{t+1} \right) - \left(\mathcal{L}_{t} + \alpha \Gamma_t \right) \\
        &\leq - \frac{1}{2}\eta_{x} K \| \nabla \Phi (\x_t) \|^2 + \left[a_1 + a_2 4 L_f^2 K^2 \left(\eta_x^2 + \eta_y^2\right) \right] \left[ 10 K^2 \left(\eta_{x, l}^2 \sigma_x^2 + \eta_{y, l}^2 \sigma_y^2\right) \right] + a_2 \frac{9K}{m} \left(\eta_x^2 \sigma_x^2 + \eta_y^2  \sigma_y^2 \right)\\
        &\leq - \frac{1}{2}\eta_{x} K \| \nabla \Phi (\x_t) \|^2 + \left[ \frac{1}{2}\left(L + \frac{L_f}{10}\right) + 2\right] \frac{9K}{m} \left(\eta_x^2 \sigma_x^2 + \eta_y^2  \sigma_y^2 \right)\\
        &\quad + \left[K L_f^2 \left( \frac{31}{20} \eta_{x} + \frac{1}{20} \eta_y \right) + \left[ \frac{1}{2}\left(L + \frac{L_f}{10}\right) + 2\right] 4 L_f^2 K^2 \left(\eta_x^2 + \eta_y^2\right) \right] \left[ 10 K^2 \left(\eta_{x, l}^2 \sigma_x^2 + \eta_{y, l}^2 \sigma_y^2\right) \right]
    \end{align*}
    Note that $\Gamma_0 = 0$.

    Telescoping and rearranging, we have:
    \begin{align*}
        &\frac{1}{T} \sum_{t=0}^{T-1} \mb{E} \| \nabla \Phi (\x_t) \|^2 \leq \frac{2 \left(\mathcal{L}_0 - \mathcal{L}_{*} \right)}{\eta_x K T} + \left[ \left(L + \frac{L_f}{10}\right) + 4\right] \frac{9}{m \eta_x} \left(\eta_x^2 \sigma_x^2 + \eta_y^2  \sigma_y^2 \right)\\
        &\quad + \left[L_f^2 \left( \frac{31}{20} + \frac{1}{20} \frac{\eta_y}{\eta_{x}} \right) + \left[ \frac{1}{2}\left(L + \frac{L_f}{10}\right) + 2\right] 4 L_f^2 K \left(\eta_x + \frac{\eta_y^2}{\eta_x}\right) \right] \left[ 20 K^2 \left(\eta_{x, l}^2 \sigma_x^2 + \eta_{y, l}^2 \sigma_y^2\right) \right].
    \end{align*}
    
\end{proof}
% !TEX root = main.tex

\subsection{Proof for \algn Option II} \label{subsec: fedvr}
For \algn Option II, the update rule is:
\begin{align*}
    \x_{t+1} &= \x_t - \eta_{x, g} \eta_{x, l} \left[\frac{1}{m} \sum_{i \in S_t} \sum_{j \in [K]} \left( \nabla_x f_i(\z_{t, i}^j, \xi_{t,i}^j) - \nabla_x f_i(\z_t, \xi_{t,i}) + \frac{1}{m} \sum_{i \in S_t} \nabla_x f_i(\z_t, \xi_{t,i}) \right) \right],  \\
    \y_{t+1} &= \y_t + \eta_{y, g} \eta_{y, l} \left[\frac{1}{m} \sum_{i \in S_t} \sum_{j \in [K]} \left( \nabla_y f_i(\z_{t, i}^j, \xi_{t,i}^j) - \nabla_y f_i(\z_t, \xi_{t,i}) + \frac{1}{m} \sum_{i \in S_t} \nabla_y f_i(\z_t, \xi_{t,i}) \right) \right], \\
    \e_{x, t} &= \frac{1}{mK} \sum_{i \in S_t} \sum_{j \in [K]} \left[ \nabla_x f_i(\z_t) - \left( \nabla_x f_i(\z_{t, i}^j, \xi_{t,i}^j) - \nabla_x f_i(\z_t, \xi_{t,i}) + \frac{1}{m} \sum_{i \in S_t} \nabla_x f_i(\z_t, \xi_{t,i}) \right) \right] \\
    &= \frac{1}{mK} \sum_{i \in S_t} \sum_{j \in [K]} \left[ \nabla_x f_i(\z_t) - \nabla_x f_i(\z_{t, i}^j, \xi_{t,i}^j) \right], \\
    \bar{\e}_{x, t} &= \mb{E} [\e_{x, t}] = \frac{1}{mK} \sum_{i \in S_t} \sum_{j \in [K]} \left( \nabla_x f_i(\z_t) - \nabla_x f_i(\z_{t, i}^j) \right), \\
    \e_{y, t} &= \frac{1}{mK} \sum_{i \in S_t} \sum_{j \in [K]} \left[ \nabla_y f_i(\z_t) - \left( \nabla_y f_i(\z_{t, i}^j, \xi_{t,i}^j) - \nabla_y f_i(\z_t, \xi_{t,i}) + \frac{1}{m} \sum_{i \in S_t} \nabla_y f_i(\z_t, \xi_{t,i}) \right) \right] \\
    &= \frac{1}{mK} \sum_{i \in S_t} \sum_{j \in [K]} \left( \nabla_y f_i(\z_t) - \nabla_y f_i(\z_{t, i}^j, \xi_{t,i}^j) \right), \\
    \bar{\e}_{y, t} &= \mb{E} [\e_{y, t}] = \frac{1}{mK} \sum_{i \in S_t} \sum_{j \in [K]} \left( \nabla_y f_i(\z_t) - \nabla_y f_i(\z_{t, i}^j) \right).
\end{align*}

\begin{restatable} {lemma} {stochasticgradient_vr}
    \label{lemma:sg_vr}
        \begin{eqnarray*}
            \mb{E}\| \left(\u_{x, t} - \e_{x, t} \right) \|^2 \leq \frac{4}{MK} \sum_{i \in [M]} \sum_{j \in [K]} \bigg[ L_f^2 \mb{E} \left\| \z_{t, i}^j - \z_t \right\|^2 + \left\| \nabla_x f(\z_t) \right\|^2 \bigg] + \frac{9}{mK} \sigma_x^2, \\
            \mb{E}\| \left(\u_{y, t} - \e_{y, t} \right) \|^2 \leq \frac{4}{MK} \sum_{i \in [M]} \sum_{j \in [K]} \bigg[ L_f^2 \mb{E} \left\| \z_{t, i}^j - \z_t \right\|^2 + \left\| \nabla_y f(\z_t) \right\|^2 \bigg] + \frac{9}{mK} \sigma_y^2.
        \end{eqnarray*}
\end{restatable}
\begin{proof}
    \begin{align*}
        &\mb{E} \|\left(\u_{x, t} - \e_{x, t} \right) \|^2 
        \leq \mb{E} \left\| \frac{1}{mK} \sum_{i \in S_t} \sum_{j \in [K]} \left[\nabla_x f_i(\z_{t, i}^j) - \mb{E}[\v_{x, t}^i] + \mb{E}[\bar{\v}_{x, t}] \right] \right\|^2 + \frac{9}{mK} \sigma_x^2 \\
        &\leq \frac{4}{MK} \sum_{i \in [M]} \sum_{j \in [K]} \bigg[ \mb{E} \left\| \nabla_x f_i(\z_{t, i}^j) - \nabla_x f_i(\z_t) \right\|^2 + \mb{E} \left\| \mb{E}[\v_{x, t}^i] - \nabla_x f_i(\z_t) \right\|^2 \\
        &\quad + \mb{E} \left\| \mb{E}[\bar{\v}_{x, t}] - \nabla_x f(\z_t) \right\|^2 + \left\| \nabla_x f(\z_t) \right\|^2 \bigg] + \frac{9}{mK} \sigma_x^2 \\
        &\leq \frac{4}{MK} \sum_{i \in [M]} \sum_{j \in [K]} \bigg[ L_f^2 \mb{E} \left\| \z_{t, i}^j - \z_t \right\|^2 + \left\| \nabla_x f(\z_t) \right\|^2 \bigg] + \frac{9}{mK} \sigma_x^2,
        % &\leq 4 \mb{E} \| \bar{\e}_{x, t} \|^2 + 4 \left\| \nabla_x f(\z_t) \right\|^2 + \frac{6}{mK} \sigma_x^2
    \end{align*}
    where the last inequality is due to $\mb{E}[\v_{x, t}^i] = \nabla_x f_i(\z_t)$ and $\mb{E}[\bar{\v}_{x, t}] = \nabla_x f(\z_t)$, and the second inequality is due to Lemma 4 in \cite{Karimireddy2020SCAFFOLD}).

    The bound of $\| \left(\u_{y, t} - \e_{y, t} \right) \|^2$ follows from the similar proof.
\end{proof}

\begin{restatable} [Bounded Error for \algn Option II] {lemma} {error}
    \label{lemma:error_vr}
        \begin{eqnarray*}
            \mb{E} \| \bar{\e}_{x, t} \|^2 &\leq \frac{L_f^2}{MK} \sum_{i \in [M], j \in [K]} \mathbb{E} \left\| \left( \z_t - \z_{t, i}^j \right) \right\|^2, \\
            \mb{E} \| \bar{\e}_{y, t} \|^2 &\leq \frac{L_f^2}{MK} \sum_{i \in [M], j \in [K]} \mathbb{E} \left\| \left( \z_t - \z_{t, i}^j \right) \right\|^2.
        \end{eqnarray*}
\end{restatable}
\begin{proof}
    \begin{align*}
        \mathbb{E} \| \bar{\e}_{x, t} \|^2 &= \mathbb{E} \left\| \frac{1}{mK} \sum_{i \in S_t} \sum_{j \in [K]} \left( \nabla_x f_i(\z_t) - \nabla_x f_i(\z_{t, i}^j) \right) \right\|^2 \\
        &\leq \frac{1}{mK}  \mathbb{E} \left[\sum_{i \in S_t} \sum_{j \in [K]} \left\| \left( \nabla_x f_i(\z_t) - \nabla_x f_i(\z_{t, i}^j) \right) \right\|^2 \right] \\
        &\leq \frac{L_f^2}{MK} \sum_{i \in [M], j \in [K]} \mathbb{E} \left\| \left( \z_t - \z_{t, i}^j \right) \right\|^2.
    \end{align*}
    $\mb{E} \| \bar{\e}_{y, t} \|^2$ has the same bounds.
\end{proof}

\begin{restatable} [Local Step Distance for \algn Option II] {lemma} {distance}
    \label{lemma:distance_vr}
    $\forall i \in [M], j \in [K]$, we can bound the local step distance as follows:
    \begin{align*}
        &\mb{E} \left\| \left( \z_t - \z_{t, i}^{j} \right) \right\|^2 \\
        &\leq 5K \left(16 K + 1\right) \eta_{x, l}^2 \sigma_x^2 + 5K \left(16 K + 1\right) \eta_{y, l}^2 \sigma_y^2 + 40K^2 \left(\eta_{x, l}^2 \mb{E}\left\| \nabla_x f(\z_t) \right\|^2 + \eta_{y, l}^2 \mb{E} \left\| \nabla_y f(\z_t) \right\|^2 \right).
    \end{align*}
\end{restatable}

\begin{proof}
    \begin{align*}
        &\mb{E} \left\| \left( \z_t - \z_{t, i}^{j+1} \right) \right\|^2
        = \mb{E} \left[ \left\| \x_{t, i}^j - \x_t - \eta_{x, l} \left( \nabla_x f_i (\z_{t, i}^j, \xi_{t,i}^j) - \v_{x, t}^i + \bar{\v}_{x, t} \right) \right\|^2 \right] \\
        &\quad + \mb{E} \left[ \left\| \y_{t, i}^j - \y_t + \eta_{y, l} \left(\nabla_y f_i (\z_{t, i}^j, \xi_{t,i}^j) - \v_{x, t}^i + \bar{\v}_{y, t} \right) \right\|^2 \right] \\
        &= \mb{E} \left[ \left\| \x_{t, i}^j - \x_t - \eta_{x, l} \left( \nabla_x f_i (\z_{t, i}^j) - \v_{x, t}^i + \bar{\v}_{x, t} \right) \right\|^2 \right] + \eta_{x, l}^2 \sigma_x^2 \\
        &\quad + \mb{E} \left[ \left\| \y_{t, i}^j - \y_t + \eta_{y, l} \left(\nabla_y f_i (\z_{t, i}^j) - \v_{x, t}^i + \bar{\v}_{y, t} \right) \right\|^2 \right] + \eta_{y, l}^2 \sigma_y^2 \\
        &= \left(1 + \frac{1}{2K - 1}\right) \mb{E} \left\| \x_{t, i}^j - \x_t \right\|^2 + 2K \mb{E} \left\| \eta_{x, l} \left( \nabla_x f_i (\z_{t, i}^j) - \v_{x, t}^i + \bar{\v}_{x, t} \right) \right\|^2 + \eta_{x, l}^2 \sigma_x^2 \\
        &\quad + \left(1 + \frac{1}{2K - 1}\right) \mb{E} \left\| \y_{t, i}^j - \y_t \right\|^2 + 2K \mb{E} \left\| \eta_{y, l} \left(\nabla_y f_i (\z_{t, i}^j) - \v_{x, t}^i + \bar{\v}_{y, t} \right) \right\|^2 + \eta_{y, l}^2 \sigma_y^2 \\
        &= \left(1 + \frac{1}{2K - 1}\right) \mb{E} \left\| \z_{t, i}^j - \z_t \right\|^2 + 2K \mb{E} \left\| \eta_{x, l} \left( \nabla_x f_i (\z_{t, i}^j) - \v_{x, t}^i + \bar{\v}_{x, t} \right) \right\|^2 + \eta_{x, l}^2 \sigma_x^2 \\
        &\quad + 2K \mb{E} \left\| \eta_{y, l} \left(\nabla_y f_i (\z_{t, i}^j) - \v_{x, t}^i + \bar{\v}_{y, t} \right) \right\|^2 + \eta_{y, l}^2 \sigma_y^2 \\
        &\leq \left(1 + \frac{1}{2K - 1}\right) \mb{E} \left\| \z_{t, i}^j - \z_t \right\|^2 + 2K \eta_{x, l}^2 \left[4 L_f^2 \mb{E} \left\| \z_{t, i}^j - \z_t \right\|^2 + 8 \sigma_x^2 + 4 \mb{E} \left\| \nabla_x f(\z_t) \right\|^2\right] \\
        & \quad + 2K \eta_{y, l}^2 \left[4 L_f^2 \mb{E} \left\| \z_{t, i}^j - \z_t \right\|^2 + 8 \sigma_y^2 + 4 \mb{E} \left\| \nabla_y f(\z_t) \right\|^2\right] + \eta_{x, l}^2 \sigma_x^2 + \eta_{y, l}^2 \sigma_y^2 \\
        &\leq \left(1 + \frac{1}{2K - 1} + 8K \max \{L_f^2 \eta_{x, l}^2 , L_f^2 \eta_{y, l}^2\}\right) \mb{E} \left\| \z_{t, i}^j - \z_t \right\|^2 + \left(16 K + 1\right) \eta_{x, l}^2 \sigma_x^2 \\
        &\quad + \left(16 K + 1\right) \eta_{y, l}^2 \sigma_y^2 + 8K \left(\eta_{x, l}^2 \mb{E}\left\| \nabla_x f(\z_t) \right\|^2 + \eta_{y, l}^2 \mb{E} \left\| \nabla_y f(\z_t) \right\|^2 \right) \\
        &\leq \left(1 + \frac{1}{K - 1}\right) \mb{E} \left\| \z_{t, i}^j - \z_t \right\|^2 + \left(16 K + 1\right) \eta_{x, l}^2 \sigma_x^2 \\
        &\quad + \left(16 K + 1\right) \eta_{y, l}^2 \sigma_y^2 + 8K \left(\eta_{x, l}^2 \mb{E}\left\| \nabla_x f(\z_t) \right\|^2 + \eta_{y, l}^2 \mb{E} \left\| \nabla_y f(\z_t) \right\|^2 \right) \\
        &\leq \sum_{\tau = 0}^{j-1} \left(1 + \frac{1}{K - 1} \right)^\tau \bigg[ \left(16 K + 1\right) \eta_{x, l}^2 \sigma_x^2 + \left(16 K + 1\right) \eta_{y, l}^2 \sigma_y^2 \\
        &\quad + 8K \left(\eta_{x, l}^2 \mb{E}\left\| \nabla_x f(\z_t) \right\|^2 + \eta_{y, l}^2 \mb{E} \left\| \nabla_y f(\z_t) \right\|^2 \right) \bigg] \\
        &\leq 5K \left(16 K + 1\right) \eta_{x, l}^2 \sigma_x^2 + 5K \left(16 K + 1\right) \eta_{y, l}^2 \sigma_y^2 + 40K^2 \left(\eta_{x, l}^2 \mb{E}\left\| \nabla_x f(\z_t) \right\|^2 + \eta_{y, l}^2 \mb{E} \left\| \nabla_y f(\z_t) \right\|^2 \right),
    \end{align*}
$\bar{\v}_{x, t} = \frac{1}{m} \sum_{i \in S_t} \nabla_x f_i(\z_t, \xi_{t, i}) $ and $\v_{x, t}^i = \nabla_x f_i (\z_t, \xi_{t,i})$; 
$\bar{\v}_{y, t} = \frac{1}{m} \sum_{i \in S_t} \nabla_y f_i(\z_t, \xi_{t, i})$ and $\bar{\v}_{y, t}^i = \nabla_y f_i (\z_t, \xi_{t,i})$; 
where the first inequality is due to bounded variance of stochastic gradient, the second and third inequalities follow from the fact $\| \a + \b \|^2 \leq \left(1 + \frac{1}{\epsilon}\right) \| \a \|^2 + \left(1 + \epsilon \right) \| \b \|^2$, the forth inequality is due to smoothness of $f$ in $x$ and $y$, fifth inequality holds if 
\begin{align}
    4K \max \{L_f^2 \eta_{x, l}^2 , L_f^2 \eta_{y, l}^2\} \leq \frac{1}{2(K-1)(2K-1)}, \label{sagda2_lr1}
\end{align}
and the last inequality follows from the $\sum_{\tau = 0}^{j-1} \left(1 + \frac{1}{K - 1} \right)^\tau \leq (K - 1) \left[ \left(1 + \frac{1}{K - 1} \right)^K - 1 \right] \leq 5K$.

\begin{align*}
    &\mb{E} \left\| \left( \nabla_x f_i (\z_{t, i}^j) - \v_{x, t}^i + \bar{\v}_{x, t} \right) \right\|^2 \\
    &= \mb{E} \left\| \left( \nabla_x f_i (\z_{t, i}^j) - \nabla_x f_i(\z_t) \right) + \left( \nabla_x f_i(\z_t) - \v_{x, t}^i \right) + \left( \bar{\v}_{x, t} - \nabla_x f(\z_t) \right) + \nabla_x f(\z_t) \right\|^2 \\
    &\leq 4 \mb{E} \left\| \nabla_x f_i (\z_{t, i}^j) - \nabla_x f_i(\z_t) \right\|^2 + 4 \mb{E} \left\| \nabla_x f_i(\z_t) - \v_{x, t}^i \right\|^2 + 4 \mb{E} \left\| \bar{\v}_{x, t} - \nabla_x f(\z_t) \right\|^2 + 4 \mb{E} \left\| \nabla_x f(\z_t) \right\|^2 \\
    &\leq 4 L_f^2 \mb{E} \left\| \z_{t, i}^j - \z_t \right\|^2 + 8 \sigma_x^2 + 4 \mb{E} \left\| \nabla_x f(\z_t) \right\|^2
\end{align*}

\end{proof}

\begin{proof}
    Similar to the bound of $\Phi$ and $f$ in \eqref{lemma:phi} and \eqref{lemma:f}, we have the following results:
    \begin{align*}
        \mb{E} \Phi (\x_{t+1}) - \Phi (\x_t)
        &\leq - \frac{1}{2}\eta_{x} K \| \nabla \Phi (\x_t) \|^2 - \frac{1}{4} \eta_{x} K \| \nabla_x f(\z_t) \|^2 + \frac{3}{2} \eta_{x} K \mb{E}\left\| \bar{\e}_{x, t} \right\|^2 \\
        &\quad + \eta_{x} K \frac{L_f^2}{\mu^2} \left\| \nabla_y f(\z_t) \right\|^2 + \frac{1}{2} L \eta_x^2 K^2 \mb{E}\left\| \u_{x,t} - \e_{x, t} \right\|^2.
    \end{align*}

    \begin{align*}
        f(\z_t) - \mb{E} f(\z_{t+1})
        &\leq \frac{3}{2} \eta_x K \left\| \nabla_x f(\z_t) \right\|^2 + \frac{1}{2} \eta_x K \mb{E} \left\| \bar{\e}_{x, t} \right\|^2 + \frac{1}{2} \eta_y K \mb{E}\left\| \bar{\e}_{y, t} \right\|^2 - \frac{1}{2} \eta_y K \| \nabla_y f(\z_t) \|^2 \\
        &\quad + \frac{1}{2} L_f \eta_x^2 K^2 \left\| \u_{x, t} - \e_{x, t} \right\|^2 + \frac{1}{2} L_f \eta_y^2 K^2 \left\| \u_{y, t} - \e_{y, t} \right\|^2.
    \end{align*}

    Define potential function $\mathcal{L}_t = \Phi (\x_t) - \frac{1}{10} f(\z_t)$, 
    \begin{align*}
        &\mb{E} \mathcal{L}_{t+1} - \mathcal{L}_{t} 
        = \mb{E} \Phi (\x_{t+1}) - \Phi (\x_t) + \frac{1}{10} \left(f(\z_t) - \mb{E} f(\z_{t+1})\right) \\
        &\leq - \frac{1}{2}\eta_{x} K \| \nabla \Phi (\x_t) \|^2 
        - \frac{1}{10} \eta_{x} K \left\| \nabla_x f(\z_t) \right\|^2 - \eta_y K \left(\frac{1}{20} - \frac{\eta_{x}}{\eta_y} \frac{L_f^2}{\mu^2}\right) \left\| \nabla_y f(\z_t) \right\|^2 \\
        &\quad + \frac{31}{20} \eta_{x} K \left\| \bar{\e}_{x, t} \right\|^2 + \frac{1}{20} \eta_y K \left\| \bar{\e}_{y, t} \right \|^2 \\
        &\quad + \frac{1}{2}\left(L + \frac{L_f}{10}\right) \eta_x^2 K^2 \mb{E} \left\| \u_{x, t} - \e_{x, t} \right\|^2 + \frac{1}{20} L_f \eta_y^2 K^2 \mb{E}\left\| \u_{y, t} - \e_{y, t} \right\|^2 \\
        &\leq - \frac{1}{2}\eta_{x} K \| \nabla \Phi (\x_t) \|^2 
        - \frac{1}{10} \eta_{x} K \left\| \nabla_x f(\z_t) \right\|^2 - \eta_y K \left(\frac{1}{20} - \frac{\eta_{x}}{\eta_y} \frac{L_f^2}{\mu^2}\right) \left\| \nabla_y f(\z_t) \right\|^2 \\
        &\quad + \left(\frac{31}{20} \eta_{x} K + \frac{1}{20} \eta_y K \right) \left[\frac{L_f^2}{MK} \sum_{i \in [M], j \in [K]} \mathbb{E} \left\| \left( \z_t - \z_{t, i}^j \right) \right\|^2\right] \\
        &\quad + \frac{1}{2}\left(L + \frac{L_f}{10}\right) \eta_x^2 K^2 \left[\frac{4}{MK} \sum_{i \in [M]} \sum_{j \in [K]} \bigg[ L_f^2 \mb{E} \left\| \z_{t, i}^j - \z_t \right\|^2 + \left\| \nabla_x f(\z_t) \right\|^2 \bigg] + \frac{9}{mK} \sigma_x^2\right] \\
        &\quad + \frac{1}{20} L_f \eta_y^2 K^2 \left[\frac{4}{MK} \sum_{i \in [M]} \sum_{j \in [K]} \bigg[ L_f^2 \mb{E} \left\| \z_{t, i}^j - \z_t \right\|^2 + \left\| \nabla_y f(\z_t) \right\|^2 \bigg] + \frac{9}{mK} \sigma_y^2\right] \\
        &\leq - \frac{1}{2}\eta_{x} K \| \nabla \Phi (\x_t) \|^2 
        - \frac{1}{10} \eta_{x} K \left\| \nabla_x f(\z_t) \right\|^2 - \eta_y K \left(\frac{1}{20} - \frac{\eta_{x}}{\eta_y} \frac{L_f^2}{\mu^2}\right) \left\| \nabla_y f(\z_t) \right\|^2 \\
        &\quad + \underbrace{L_f^2 \left[\frac{31}{20} \eta_{x} K + \frac{1}{20} \eta_y K + 2\left(L + \frac{L_f}{10}\right) \eta_x^2 K^2 + \frac{1}{5} L_f \eta_y^2 K^2 \right]}_{a_1} \left[\frac{1}{MK} \sum_{i \in [M], j \in [K]} \mathbb{E} \left\| \left( \z_t - \z_{t, i}^j \right) \right\|^2\right] \\
        &\quad + \underbrace{2\left(L + \frac{L_f}{10}\right) \eta_x^2 K^2}_{a_2} \left\| \nabla_x f(\z_t) \right\|^2 + \frac{1}{2}\left(L + \frac{L_f}{10}\right) \eta_x^2 K^2 \frac{9}{mK} \sigma_x^2 \\
        &\quad + \underbrace{\frac{1}{5} L_f \eta_y^2 K^2}_{a_3} \left\| \nabla_y f(\z_t) \right\|^2 + \frac{1}{20} L_f \eta_y^2 K^2 \frac{9}{mK} \sigma_y^2 \\
        &\leq - \frac{1}{2}\eta_{x} K \| \nabla \Phi (\x_t) \|^2 
        - \frac{1}{10} \eta_{x} K \left\| \nabla_x f(\z_t) \right\|^2 - \eta_y K \left(\frac{1}{20} - \frac{\eta_{x}}{\eta_y} \frac{L_f^2}{\mu^2}\right) \left\| \nabla_y f(\z_t) \right\|^2 \\
        &\quad + \left[ 5K \left(16 K + 1\right) \eta_{x, l}^2 a_1 + \frac{1}{2}\left(L + \frac{L_f}{10}\right) \eta_x^2 \frac{9K}{m} \right] \sigma_x^2 + \left[ 5K \left(16 K + 1\right) \eta_{y, l}^2 a_1 + \frac{1}{20} L_f \eta_y^2 \frac{9K}{m} \right] \sigma_y^2  \\
        &\quad + \left(a_2 + 40K^2 \eta_{x, l}^2 a_1\right) \left\| \nabla_x f(\z_t) \right\|^2 + \left(a_3 + 40K^2 \eta_{y, l}^2 a_1 \right) \left\| \nabla_y f(\z_t) \right\|^2  \\
        &\leq - \frac{1}{2}\eta_{x} K \| \nabla \Phi (\x_t) \|^2 + \left[ 5K \left(16 K + 1\right) \eta_{x, l}^2 a_1 + \frac{1}{2}\left(L + \frac{L_f}{10}\right) \eta_x^2 \frac{9K}{m} \right] \sigma_x^2 \\
        &\quad + \left[ 5K \left(16 K + 1\right) \eta_{y, l}^2 a_1 + \frac{1}{20} L_f \eta_y^2 \frac{9K}{m} \right] \sigma_y^2
    \end{align*}
    where the last inequality follows from the conditions:
    \begin{align}
        & \frac{1}{10} \eta_x K - \left(a_2 + 40K^2 \eta_{x, l}^2 a_1\right) \geq 0, \label{sagda2_lr2} \\
        & \eta_y K \left(\frac{1}{20} - \frac{\eta_{x}}{\eta_y} \frac{L_f^2}{\mu^2}\right) - \left(a_3 + 40K^2 \eta_{y, l}^2 a_1 \right) \geq 0. \label{sagda2_lr3}
    \end{align}
    
    Telescoping and rearranging, we have:
    \begin{align*}
        &\frac{1}{T} \sum_{t=0}^{T-1} \mb{E} \| \nabla \Phi (\x_t) \|^2 \leq \frac{2 \left(\mathcal{L}_0 - \mathcal{L}_{*} \right)}{\eta_x K T} + \left[ 10 \left(16 K + 1\right) \eta_{x, l}^2 \frac{a_1}{\eta_x} + \left(L + \frac{L_f}{10}\right) \frac{9 \eta_x}{m} \right] \sigma_x^2 \\
        &\quad + \left[ 10 \left(16 K + 1\right) \eta_{y, l}^2 \frac{a_1}{\eta_x} + \frac{9}{10} L_f \frac{\eta_y^2}{m \eta_x} \right] \sigma_y^2 \\
        &\leq \frac{2 \left(\mathcal{L}_0 - \mathcal{L}_{*} \right)}{\eta_x K T} + \left[\left(L + \frac{L_f}{10}\right) \frac{9 \eta_x}{m} \sigma_x^2 + + \frac{9}{10} L_f \frac{\eta_y^2}{m \eta_x}  \sigma_y^2 \right] \\
        &\quad + L_f^2 \left[\frac{31}{20} K + \frac{1}{20} \frac{\eta_y}{\eta_x} K + 2\left(L + \frac{L_f}{10}\right) \eta_x K^2 + \frac{1}{5} L_f \frac{\eta_y^2}{\eta_x} K^2 \right] \left[ 10 \left(16 K + 1\right) \right] \left(\eta_{x, l}^2 \sigma_x^2 + \eta_{y, l}^2 \sigma_y^2 \right).
    \end{align*}
    
\end{proof}

\end{document}